\documentclass[15pt]{article}

\usepackage[numbers]{natbib}

\usepackage{microtype}
\usepackage{graphicx}
\usepackage{subfigure}
\usepackage{booktabs}
\usepackage{amsmath}
\usepackage{algorithm}
\usepackage{algorithmic}
\usepackage{rotating}
\usepackage{changepage}
\usepackage[thmmarks,amsmath]{ntheorem}
\usepackage{multirow}
\usepackage{amsfonts,amssymb}
\usepackage[T1]{fontenc}
\usepackage{url}
\usepackage{epstopdf}

\newcommand{\dataset}{{\cal D}}

\newcommand{\X}{{\mathcal X}}
\newcommand{\Y}{{\mathcal Y}}

\newcommand{\R}{{\mathbb R}}

\newcommand{\be}{\begin{eqnarray}}
\newcommand{\ben}{\begin{eqnarray*}}
\newcommand{\en}{\end{eqnarray}}
\newcommand{\enn}{\end{eqnarray*}}

\newtheorem{theorem}{Theorem}
\newtheorem{lemma}{Lemma}

\newenvironment{proof}{{\em Proof.}}{$\Box$}

\usepackage{lineno,hyperref}

\title{A Minimax Probability Machine for\\
Non-Decomposable Performance Measures}
\date{}
\author{Junru~Luo 
\thanks{J. Luo is with the School of Computer Science and Artificial Intelligence $\&$ Aliyun School of Big Data,
Changzhou University, Changzhou, Jiangsu province, China e-mail: (luojunru@cczu.edu.cn). } 
        , Hong~Qiao
\thanks{H. Qiao is the Institute of Automation, Chinese Academy of Sciences, Beijing 100190, China
and School of Artificial Intelligence, University of Chinese Academy of Sciences, Beijing 100049, China
e-mail: (hong.qiao@ia.ac.cn).}
        , and~Bo~Zhang
\thanks{B. Zhang is the Academy of Mathematics and Systems Science, Chinese Academy of Sciences, Beijing 100190, China
and School of Mathematical Sciences, University of Chinese Academy of Sciences, Beijing 100049, China
e-mail: (b.zhang@amt.ac.cn).}}

\begin{document}
\bibliographystyle{plain}
\maketitle

\begin{abstract}
Imbalanced classification tasks are widespread in many real-world applications.
For such classification tasks, in comparison with the accuracy rate, it is usually much more
appropriate to use non-decomposable performance measures such as the Area Under the receiver operating
characteristic Curve (AUC) and the $F_\beta$ measure as the classification criterion since the label class
is imbalanced. On the other hand, the minimax probability machine is a popular method for binary
classification problems and aims at learning a linear classifier by maximizing the accuracy rate,
which makes it unsuitable to deal with imbalanced classification tasks.
The purpose of this paper is to develop a new minimax probability machine for the $F_\beta$ measure,
called MPMF, which can be used to deal with imbalanced classification tasks.
A brief discussion is also given on how to extend the MPMF model for several other non-decomposable
performance measures listed in the paper.
To solve the MPMF model effectively, we derive its equivalent form which can then be solved
by an alternating descent method to learn a linear classifier.
Further, the kernel trick is employed to derive a nonlinear MPMF model to learn a nonlinear classifier.
Several experiments on real-world benchmark datasets demonstrate the effectiveness of our new model.
\end{abstract}




\section{Introduction}

Binary classification may be the most encountered problems in real-world applications and has
been extensively studied.
The standard binary model seeks to learn a classifier by maximizing the Accuracy Rate (AR)
or minimizing the error rate. However, the accuracy rate is not an appropriate metric in the setting
where the label class is imbalanced \cite{Musicant2003,Dembczynsk2011,Menon2013,Parambath2014}.
A simple majority algorithm can guarantee a high prediction accuracy in such a case.
Many algorithms were devised to deal with imbalanced classification problems,
based on the cost-sensitive learning and resampling techniques \cite{Chawla2002,Khan2018,He2009,Kang2018,Kraw2020,Mathew2018}.
A comprehensive survey on related works can be found in \cite{He2009,Johnson2019,Guo2017}.
The classical imbalanced learning algorithms aim at maximizing a modified accuracy rate.
Several other evaluation metrics have been proposed to measure the learned classifier, including AUC,
$F_\beta$-measure, and the geometric mean of the True Positive Rate (TPR) and 
the True Negative Rate (TNR) \cite{Hu2018,Cano2013,Zhu2020}, 
which are more appropriate for imbalanced classification tasks.
But, different from the accuracy rate, these measures can not be expressed as a sum of independent
metrics on individual examples and thus are called non-decomposable performance measures.
They give a more holistic evaluation of the entire data, which makes them challenging to be optimized.

In recent years, many methods have been proposed which focus on optimizing non-decomposable performance
measures \cite{Bascol2019,Sanyal2018,Koyejo2014,Eban2017,Narasimhan2019}.
By \cite{Ye2012}, these methods fall into two groups: the Empirical Utility Maximization (EUM) and
the Decision-Theoretic Approach (DTA).

EUM is also called the Population Utility (PU) in \cite{Dembczynsk2017}.
It learns a classifier by maximizing the corresponding empirical measures based on the training data
and then predicts an unseen instance by using the learned classifier.
Several works have applied the traditional algorithms, such as SVM and logistic regression, to the
non-decomposable performance measure maximization by replacing the error rate with a specially designed
surrogate loss function \cite{Musicant2003,Jansche2005,Joachims2005,Chinta2013}.
Recently, studies are focused on the plug-in approach, which learns a class probability function first
based on a training example and then makes a decision of the threshold according to another data
set \cite{Menon2013,Ye2012,Kotlowski2017}.
\cite{Narasimhan2014} provided a general methodology to show the statistical consistency of the plug-in
approach for any performance measure that can be expressed as a continuous function of the true positive
and true negative rates as well as the class proportion, such as $F_\beta$ and the Geometric Mean (GM) of
TPR and TNR. Cost-sensitive classification is also used in this approach, which learns a classifier by
minimizing a weighted False Positive Rate (FPR) and False Negative Rate (FNR).
\cite{Menon2013} proved the consistency of the empirically cost-sensitive learning algorithm concerning
the Arithmetic Mean (AM) of the TPR and TNR metrics, in which the cost parameter is determined from the
empirical class ratio.
A cost-sensitive Support Vector Machine (SVM) with a tuning threshold was proposed in \cite{Parambath2014} 
for the metric $F_\beta$, where the cost parameter is searched with a grid-based method.
\cite{Bascol2019} presented the Cone-based Optimal Next Evaluation (CONE) algorithm that iteratively
selects the classification costs that lead to a near-optimal $F_\beta$-measure.
As discussed in Section \ref{sec2}, many non-decomposable performance measures except AM aim at minimizing
a polynomial combination of FPR and FNR rather than a linear combination of them.
The cost-sensitive learning approach seems not an appropriate method for these metrics.

DTA was first proposed in \cite{Lewis1995}.
It considers set classifiers and predicts the label of a test set by maximizing its expected measure.
Thus it needs the entire joint distribution of the ground-truth which can be estimated by
using the maximum likelihood principle or the logistic regression model.
Several methods have been proposed to solve this type of problems, based on the assumption that the labels
are independent \cite{Lewis1995,Jansche2007,Ye2012}.
Later, in \cite{Dembczynsk2011,Waegeman2014} this assumption was removed, and a general algorithm was given
which needs only $n^2+1$ parameters of the joint distribution, where $n$ is the sample size.
The above work focused on the $F_\beta$ metric. Recently, a general analysis of DTA was provided
in \cite{Natarajan2015} for non-decomposable performance measures.
It is also named the Expected Test Utility (ETU) in \cite{Dembczynsk2017}.

The Minimax Probability Machine (MPM) method is a competitive algorithm for binary classification problems
which was first proposed in \cite{Lanckriet2002}.
It aims at maximizing the accuracy rate of a random instance in a worst-case setting,
where only the mean and covariance matrix of each class are assumed to be known \cite{Lanckriet2002}.
It forces the worst-case probability of misclassification of each class to be exactly equal,
which is usually not true in real-world applications.
The minimum error minimax probability machine (MEMPM) removes this constraint and seeks to learn
a linear classifier by minimizing a weighted probability of misclassification of the future data
in the same worst-case setting \cite{Huang2004}. 
However, MPM and MEMPM are not suitable for imbalanced classification tasks since they are based on 
the accuracy rate which is not an appropriate metric for such tasks, as mentioned above.

In this paper, we develop an MPM model based on non-decomposable performance measures for the first time. 
This model can deal with imbalanced classification tasks where the mean and covariance
matrix of data are given or can be estimated, which is different from \cite{Lanckriet2002,Huang2004}.
We first present a minimax probability machine for the $F_\beta$-measures (MPMF) and then briefly discuss 
how to extend it to several other non-decomposable performance measures as listed in Table \ref{NDMeasures}.
Further, to solve MPMF effectively, its equivalent minimization formulation is derived 
whose objective function is given in terms of a polynomial combination of FPR and FNR.
The equivalent minimization problem is then solved by the alternative descent method.
The convergence of the method is also analyzed and verified numerically.
Furthermore, we explore the kernel trick in the setting to obtain a nonlinear MPM model, yielding 
a nonlinear classifier which can effectively deal with nonlinearly separable imbalanced problems.
Note that, due to the difficulty caused by the non-decomposability, most of the existing works related to 
non-decomposable metrics only consider linear models.
A main feature of our method is that it only makes use of the mean and covariance matrix of data and thus 
is independent of the training data size, making our method appropriate for large-scale problems.
Another feature of our method is that it has no hyper-parameters to choose.
Several experiments on real world benchmark datasets presented in the paper and compared with the plug-in
method (a state-of-the-art method) illustrate that our MPMF method is effective for imbalanced classification 
problems. 

The paper is organized as follows.
In Section \ref{sec2}, we introduce several non-decomposable performance measures for binary classification
and derive their equivalent minimization objective functions, which can be expressed as a polynomial combination
of FPR and FNR. In Section \ref{sec3}, the minimax probability machine and its variants for the accuracy rate
maximization are presented. In Section \ref{sec4}, we present the MPM model with the $F_\beta$ metric which is 
then solved by using the alternative descent method. In Section \ref{sec5}, we propose a non-linear $F_\beta$ 
maximization based on the kernel trick and the minimax probability machine. 
Several experiments are presented in Section \ref{sec6}, and conclusions are given in Section \ref{sec7}.

\begin{table*}
\caption{Non-decomposable Performance Measures: AR, AM, QM, $F_\beta$, HM, GM, G-TP/PR, JAC (first column),
their definition (second column) and their equivalent minimization objective function given in terms of 
a polynomial combination of FPR and FNR (third column)}\label{NDMeasures}
\begin{center}
\begin{small}
\begin{sc}
\resizebox{\textwidth}{!}{
\begin{tabular}{lll}
\toprule
Measure & Definition ($P_{nd}(\mbox{TPR},\mbox{TNR})$) & Minimization Objective ($Q_{nd}(\mbox{FNR},\mbox{FPR})$)\\
\midrule
AR         & $p\cdot\mbox{TPR}+(1-p)\cdot\mbox{TNR}$               & $p\cdot\mbox{FNR}+(1-p)\cdot\mbox{FPR}$ \\
AM         & $(\mbox{TPR}+\mbox{TNR})/2$                           & $(\mbox{FNR}+\mbox{FNR})/2$ \\
QM         & $1-[(1-\mbox{TPR})^2+(1-\mbox{TNR})^2]/2$             & $[(\mbox{FNR})^2+(\mbox{FPR})^2]/2$ \\
$F_\beta$  & $\frac{(1+\beta^2)p\cdot\mbox{TPR}}{(1+\beta^2)p\cdot\mbox{TPR}+(1-p)\cdot\mbox{FPR}+\beta^2p\cdot\mbox{FNR}}$
           & $\sum_{i=0}^\infty(\beta^2p\cdot\mbox{FNR}+(1-p)\cdot\mbox{FPR})(\mbox{FNR})^i$ \\
HM         & ${2\cdot\mbox{TPR}\cdot\mbox{TNR}}/[{\mbox{TPR}+\mbox{TNR}}]$  
           & $\sum_{i=0}^\infty[(\mbox{FNR})^i+(\mbox{FPR})^i]$ \\
GM         & $\sqrt{\mbox{TPR}\cdot\mbox{TNR}}$      
           & $[\sum_{i=0}^\infty(\mbox{FPR})^i][\sum_{i=0}^\infty(\mbox{FNR})^i]$ \\
G-TP/PR    & $\sqrt{\mbox{TPR}\cdot\mbox{Prescison}}$   
           & $p\sum_{i=0}^\infty(\mbox{FNR})^i +(1-p)\cdot\mbox{FPR}\sum_{i=0}^\infty[\mbox{FNR}\cdot(2-\mbox{FNR})]^i$\\
JAC        & ${p\cdot\mbox{TPR}}/[{p\cdot\mbox{TPR}+p\cdot\mbox{FNR}+(1-p)\cdot\mbox{FPR}}]$  
           & $[p\cdot\mbox{FNR}+(1-p)\cdot\mbox{FPR}]\sum_{i=0}^\infty(\mbox{FNR})^i$ \\
\bottomrule
\end{tabular}
}
\end{sc}
\end{small}
\end{center}
\end{table*}

\section{Problem Setting}\label{sec2}

In this section, we present several non-decomposable performance measures in binary classification.

\subsection{Binary Classification}

Let $\X\subset\R^d$ be the instance space and let $\Y=\{+1,-1\}$ be the label set with the distribution $\dataset$
over $\X\times\Y$. We aim to learn a classifier $f\in\mathcal{F}$, which predicts a label based on the seeing
feature and makes a minimal prediction error rate.
Given a classifier $f:\X\to\R$ and an instance $x\in\X$, the label of $x$ is assigned to be $+1$ if $f(x)>0$,
and $-1$ if otherwise. Then the binary classification problem can be expressed as
\be\label{binary}
\arg_{f\in\mathcal F}\max_{(x,y)\sim\dataset}{\bf Pr}(yf(x)>0).
\en
Note that the prior knowledge of the instance space plays a key role in machine learning problems.
Suppose the true distribution $\dataset$ is explicitly given. Then the Bayesian estimator can be computed
by solving \eqref{binary}. If the distribution class is given, we can first estimate the distribution $\dataset$
with the maximum likelihood method and then learn a classifier.
In general, the true distribution $\dataset$ or its distribution class is not known in practice.
But the moment statistics of the input space can be estimated, and thus the minimax probability
machine can be developed. In the worst-case, the sample data can be used to learn a classifier
by minimizing the empiric risk.

\subsection{Non-decomposable Performance Measures}

Let $p={\bf Pr}(y =+1)$ be the probability of the positive examples.
We now introduce some basic notations:
\ben
\mbox{TPR}&:=& 1-\mbox{FNR}= {\bf Pr}(\hat y =+1| y =+1),\\
\mbox{TNR}&:=& 1-\mbox{FPR}= {\bf Pr}(\hat y =-1| y =-1),\\
\mbox{Precision}&:=& {\bf Pr}(y =1|\hat y =1)\\
  &=&\frac{p\cdot\mbox{TPR}}{p\cdot\mbox{TPR} +(1-p)\cdot\mbox{FPR}}
\enn

A better classifier should have a greater TPR, TNR, and precision values, but there is a tradeoff among
these measures. It is unable to achieve the optimal values simultaneously for TPR, TNR, and precision
since each of these quantities can be maximized at the cost of other measures.
So usually we consider the criterion that adjusts the basic quantities, such as several combinations
of TPR and TNR. Table \ref{NDMeasures} presents some usually used performance measures.
In this paper, we focus on the $F_\beta$ measure, which is a harmonic mean of TPR and Precision and
defined as
\ben
F_\beta&=&\frac{(\beta^2+1)\cdot\mbox{TPR}\cdot\mbox{Precision}}{\beta^2\cdot\mbox{Precision}+\mbox{TPR}}\\
&=&\frac{(1+\beta^2)\cdot p\cdot\mbox{TPR}}{(1+\beta^2)\cdot p\cdot\mbox{TPR}
+(1-p)\cdot\mbox{FPR}+\beta^2p\cdot\mbox{FNR}}
\enn
A larger $F_\beta$ means a better corresponding classifier.

\begin{lemma}\label{Gbeta}
The classifier $f$ maximizes $F_\beta$ if and only if it minimizes $Q_F$, where $Q_F$ is a polynomial
combination function of FPR and FNR defined by
\ben
Q_F&=&\sum_{i=0}^\infty((1-p)\cdot\mbox{FPR}+\beta^2 p\cdot\mbox{FNR})\cdot(\mbox{FNR})^i \\
&=&\frac{(1-p)\cdot\mbox{FPR} +\beta^2 p\cdot\mbox{FNR}}{1-\mbox{FNR}}.
\enn
\end{lemma}

\begin{proof}
Since $F_\beta>0$, then maximizing $F_\beta$ is equivalent to minimizing ${1}/{F_\beta}$. We have
\ben
\frac{1}{F_\beta}
&=& 1+\frac{(1-p)\cdot\mbox{FPR}+\beta^2p\cdot\mbox{FNR}}{p(1+\beta^2)(1-\mbox{FNR})}\\
&=& 1+\sum_{i=0}^\infty\frac{[(1-p)\cdot\mbox{FPR}+\beta^2p\cdot\mbox{FNR}]\cdot(\mbox{FNR})^i}{p(1+\beta^2)}
\enn
The proof is thus complete.
\end{proof}

Note that the above method can be extended to other non-decomposable performance measures, 
such as the AM, the GM, the Quadratic Mean (QM) of TPR and TNR, the Harmonic Mean (HM) of TPR and TNR, 
the Geometric Mean (GM) of TPR and TNR, the Geometric mean of TPR and PRecision (G-TP/PR), and
the Jaccard Coefficient (JAC). Maximizing these measures is also equivalent to minimizing $Q_{nd}$,
which can also be expressed as a polynomial combination of FPR and FNR, that is, 
$Q_{nd}=\sum_{i,j=0}^\infty\tau_{ij}(\mbox{FPR})^i(\mbox{FNR})^j$
(see the third column in Table \ref{NDMeasures}).
It is interesting to note that the same minimizing objective function is obtained for both the $F_1$ measure 
and the JAC measure, which was overlooked previously in the literature.
Note further that the objective measure is replaced with a well-defined weighted linear combination of FNR 
and FPR in the cost-sensitive approach, but a polynomial combination of FRP and FNP is minimized in this paper.

\section{MPM with the Accuracy Rate Measure}\label{sec3}

Different from the traditional classification algorithms, such as SVMs and neural networks, which seek to
learn a real-valued function by minimizing the error rate on a given training data set,
MPM tries to separate the two classes of the data samples with the maximal probability in a worst-case setting
where only the mean and covariance matrix of each class are given in advance \cite{Lanckriet2002}.
Let the notation $x\sim (\mu,\Sigma)$ denote that $x$ belongs to the class of distributions having the mean $\mu$ and 
the covariance matrix $\Sigma$.
Assume that the mean and covariance matrix of the positive and negative samples, denoted, respectively,
by $\{\mu_P,\Sigma_P\}$ and $\{\mu_N,\Sigma_N\}$, are all reliable. Then MPM can be expressed
as the optimization problem
\begin{align}\label{MPM}
\begin{split}
&\max_{0<\alpha<1,w\neq 0,b}\alpha \qquad \\
\text{s.t.} &\inf_{x_P\sim(\mu_P,\Sigma_P)}{\rm\bf Pr}\{w^T x_P\ge b\}\ge\alpha\\
&\inf_{x_N\sim(\mu_N,\Sigma_N)}{\rm\bf Pr}\{w^Tx_N\le b\}\ge\alpha
\end{split}
\end{align}
where the inf is taking over all distributions with mean $\mu$ and covariance matrix $\Sigma$.
Let $(\alpha^\star,w^\star,b^\star)$ be the optimal solution of \eqref{MPM}.
It is guaranteed that the misclassification probability is less than $1-\alpha^\star$ for any future data sample
with the learned classifier $f(x)=(w^\star)^Tx-b^\star$.

To solve the optimization problem \eqref{MPM}, we need to remove the unknown probability 
in the constraint by using the following result established in \cite{Marshall1960} 
(see also Theorem 6.1 in \cite{Bertsimas2005}).

\begin{lemma}\label{Bertsimas}
Suppose $x\in\R^p $ is a random vector with $x\sim (\mu,\Sigma)$ and $S\subset\R^p$ is a given convex set. 
Then the supremum of the probability of $x\in S$ is given as
$\sup_{x\sim(\mu,\Sigma)}{\rm\bf Pr}(x\in S)={1}/({1+d^2})$,
where $d^2=\inf_{x\in S}(x-\mu)^T\Sigma^{-1}(x-\mu)$.
\end{lemma}

If the convex set $S$ is a half-space, then $d$ can be calculated explicitly, leading to
the following lemma which was proved in \cite{Lanckriet2002}.

\begin{lemma}\label{Lanckriet}
Given $w\in\R^p$ with $p\neq 0$ and $b\in\R$, the condition
\ben
\inf_{x\sim(\mu,\Sigma)}{\rm\bf Pr}\{w^T x\le b\}\ge\alpha
\enn
holds if and only if
\ben
b-w^Tx\ge\kappa(\alpha)\sqrt{w^T\Sigma w},
\enn
where $\kappa(\alpha)= \sqrt{{\alpha}/({1-\alpha})}$.
\end{lemma}

It is easy to see that $\kappa(\alpha)$ is a monotonically increasing function of $\alpha$.
By Lemma \ref{Lanckriet}, the MPM problem (\ref{MPM}) becomes
\ben
&&\max_{0<\alpha<1,w\neq0,b}\kappa(\alpha) \qquad \\
\text{s.t.} && -b+w^T\mu_P\ge\kappa(\alpha)\sqrt{w^T\Sigma_P w} \\
&& b-w^T\mu_N\ge\kappa(\alpha)\sqrt{w^T\Sigma_N w}.
\enn
Further, by eliminating $b$ the above problem reduces to
\ben
&&\max_{\gamma,w\neq0}\gamma,\\
\text{s.t.}&& w^T(\mu_P-\mu_N)\ge\gamma\left[\sqrt{w^T\Sigma_P w}+\sqrt{w^T\Sigma_N w}\right]
\enn
Without loss of generality, we may set $w^T(\mu_P-\mu_N)=1$ (see \cite{Lanckriet2002}).
As a result, MPM is finally equivalent to the following constrained second-order cone programming (SOCP):
\ben\label{SOCP}
&&\min_{w\neq0}\|\Sigma_P^{1/2}w\|_2+\|\Sigma_N^{1/2}w\|_2\\
\text{s.t.}&& w^T(\mu_P-\mu_N)=1
\enn

The above SOCP problem can be solved by the interior point method \cite{Boyd2004}.
Now, write $w=w_0+Fu$ with $w_0=({\mu_P-\mu_N}){\|\mu_P-\mu_N\|_2^{-2}}$ and $F\in\R^{p\times(p-1)}$ 
the orthogonal matrix whose columns span the subspace of the vectors orthogonal to $(\mu_P-\mu_N)$. 
Then the constraint can be eliminated to get the unconstrained SOCP \cite{Lanckriet2002}:
\ben
\min_{u\neq0}\|\Sigma_P^{1/2}(w_0+F u)\|_2+\|\Sigma_N^{1/2}(w_0+F u)\|_2,
\enn
A block coordinate descent method is then used to solve the above equivalent unconstrained
SOCP \cite{Lanckriet2002}.

In many real-world applications, the tolerance of the misclassification probability may be different
for the positive and negative classes, as seen in the disease diagnosis problem. Thus it is worth
increasing the TPR measure at the expense of a lower TNR measure, leading to
the biased minimax probability machine (BMPM) and the minimum error minimax probability
machine (MEMPM) \cite{Huang2004,Huang2004b}. MEMPM can be expressed as
\begin{align}\label{MEMPM}
\begin{split}
&\max_{0<\alpha_1<1,0<\alpha_2<1, w\neq 0,b} p\alpha_1+(1-p)\alpha_2\qquad \\
 \text{s.t.} &\inf_{x_P\sim(\mu_P,\Sigma_P)}{\rm\bf Pr}\{w^T x_P\ge b\}\ge\alpha_1 \\
&\inf_{x_N\sim(\mu_N,\Sigma_N)}{\rm\bf Pr}\{w^T x_N\le b\}\ge\alpha_2.
\end{split}
\end{align}
If $\alpha_2\in (0,1)$ is a predefined constant, then \eqref{MEMPM} becomes the BMPM problem.
The objective function in \eqref{MEMPM} is a weighted misclassification probability.
Lemma \ref{Lanckriet} can also be applied to simplify the optimization problem \eqref{MEMPM}.
Let $(\alpha_1^\star,\alpha_2^\star,w^\star,b^\star)$ be the optimal solution of \eqref{MEMPM}
and let $f(x)={w^\star}^Tx-b^\star$ be the learned classifier.
Then it is guaranteed that the obtained TPR and TNR are at least equal to $\alpha_1^\star$ and
$\alpha_2^\star$, respectively.
In MEMPM, $\alpha_1^\star$ and $\alpha_2^\star$ are not necessarily the same, which is different
from the case in MPM. Further, the experimental results obtained in \cite{Huang2004} demonstrate
the effectiveness of MEMPM.
Setting $p=1/2$ in the optimization problem \eqref{MEMPM}, a biased minimax probability machine
was derived in \cite{Huang2006} for imbalanced classification problems.
This machine tries to maximize the AM performance measure in a worst-case setting
where only the mean and covariance matrix of each class are given.
It can be regarded as a special case of our method, as will be illustrated in the next section.

\section{MPM with non-decomposable performance measures}\label{sec4}

In this section, we develop an MPM method with non-decomposable performance measures to deal with imbalanced 
classification tasks. We first consider the $F_\beta$ measure.

For an imbalanced classification problem, suppose the probability of an instance belonging to
the positive class $p$ is small enough. Then the MEMPM problem (\ref{MEMPM}) is equivalent to
maximizing $\alpha_2$ or TNR, which is not appropriate for this case.
To address this issue, we replace the objective function in \eqref{MEMPM} with the global performance
metric $F_\beta$, which is a function of TPR (or $\alpha_1$) and TNR (or $\alpha_2$). 
Thus we get the problem
\begin{align}\label{MPMFmax}
\begin{split}
&\max_{0<\alpha_1<1,0<\alpha_2<1, w\neq 0,b}F_\beta\\
\text{s.t.} &\inf_{x_P\sim(\mu_P,\Sigma_P)}{\rm\bf Pr}\{w^Tx_P\ge b\}\ge\alpha_1 \\
&\inf_{x_N\sim(\mu_N,\Sigma_N)}{\rm\bf Pr}\{w^T x_N\le b\}\ge\alpha_2
\end{split}
\end{align}
By Lemma \ref{Gbeta}, maximizing $F_\beta$ is equivalent to minimizing $Q_F$.
Set $\alpha_P=\mbox{FNR}$, $\alpha_N=\mbox{FPR}$. Then \eqref{MPMFmax} becomes
\begin{align}\label{MPMFmin}
\begin{split}
&\min_{\alpha_P,\alpha_N,w\neq 0, b}\frac{(1-p)\alpha_N+\beta^2p\alpha_P}{1-\alpha_P}\\
\text{s.t.} &\inf_{x_P\sim(\mu_P,\Sigma_P)}{\rm\bf Pr}\{w^T x_P\ge b\}\ge 1-\alpha_P \\
&\inf_{x_N\sim(\mu_N,\Sigma_N)}{\rm\bf Pr}\{w^T x_N\le b\}\ge 1-\alpha_N \\
&0<\alpha_P<1,\;0<\alpha_N<1
\end{split}
\end{align}
We call this model the minimax probability machine for the $F_\beta$-measure (MPMF).
For simplicity, define $Q_F(\alpha_P,\alpha_N;w):=[(1-p)\alpha_N+\beta^2p\alpha_P](1-\alpha_P)^{-1}$
and omit the conditions $0<\alpha_P<1,0<\alpha_N<1$ in what follows. 
Note that these two constrained conditions are guaranteed to be satisfied by the optimal value of 
$\alpha_P,\alpha_N$ obtained below (see (\ref{alphaN}) and the sentences following (\ref{solution-1}) below).
Suppose $(\alpha_P^\star,\alpha_N^\star,w^\star,b^\star)$ is the optimal solution of
the MPMF problem (\ref{MPMFmin}) and $f(x)=({w^\star})^Tx-b^\star$ is the learned classifier.
Then the $F_\beta$ metric is bigger than $({(1+\beta^2)p})[{Q^\star+(1+\beta^2)p}]^{-1}$ for any future data sample,
where $Q^\star=Q_F(\alpha_P^\star,\alpha_N^\star;w^\star)$.

We now propose an algorithm to solve the optimization problem \eqref{MPMFmin}.
By Lemma \ref{Lanckriet}, we can remove the probability terms without any distribution assumption and obtain
the following optimization problem which is equivalent to the problem \eqref{MPMFmin}:
\begin{align}\label{MPMFi}
\begin{split}
&\min_{\alpha_P,\alpha_N,w\neq 0,b}\frac{(1-p)\alpha_N+\beta^2p\alpha_P}{1-\alpha_P}\qquad \\
\text{s.t.}\;\;\; & w^T\mu_P-b\ge\pi(\alpha_P)\sqrt{w^T\Sigma_P w}\\
&b-w^T\mu_N\ge\pi(\alpha_N)\sqrt{w^T\Sigma_N w}
\end{split}
\end{align}
where $\pi(\alpha)={1}/{\kappa(\alpha)}=\sqrt{({1-\alpha})/{\alpha}}$ is monotonically decreasing with $\alpha$.

\begin{lemma}\label{equal}
The minimal value of \eqref{MPMFi} is achieved when the two inequalities in the constraints become equalities.
\end{lemma}

\begin{proof}
Assume that the optimal solution of \eqref{MPMFi} is $(w^\star,b^\star,\alpha_P^\star,\alpha_N^\star)$
and the two inequalities in the constraints hold strictly, that is,
\ben
&&(w^\star)^T\mu_P-b^\star>\pi(\alpha_P^\star)\sqrt{(w^\star)^T\Sigma_P w^\star},\\
&&b^\star-(w^\star)^T\mu_N>\pi(\alpha_N^\star)\sqrt{(w^\star)^T\Sigma_N w^\star}.
\enn
Then the objective value is getting smaller with $\alpha_P^\star$ and $\alpha_N^\star$ decreasing
while the constraints remain to hold. This is a contradiction to the fact that
$(w^\star,b^\star,\alpha_P^\star,\alpha_N^\star)$ is the optimal solution of \eqref{MPMFi}.
The proof is thus complete.
\end{proof}

By Lemma \ref{equal}, \eqref{MPMFi} can be rewritten as:
\begin{align}\label{MPMFe}
\begin{split}
&\min_{\alpha_P,\alpha_N,w\neq 0,b}\frac{(1-p)\alpha_N+\beta^2p\alpha_P}{1-\alpha_P}\qquad \\
\text{s.t.}\;\;\; & w^T\mu_P-b=\pi(\alpha_P)\sqrt{w^T\Sigma_P w}\\
& b-w^T\mu_N=\pi(\alpha_N)\sqrt{w^T\Sigma_N w}
\end{split}
\end{align}
By eliminating $b$ from \eqref{MPMFe}, the two equality constraints become
\begin{align}\label{homocondition}
\pi(\alpha_N)\sqrt{w^T\Sigma_N w}+\pi(\alpha_P)\sqrt{w^T\Sigma_P w}=w^T \mu_{pn},
\end{align}
where $\mu_{pn}=\mu_P-\mu_N$. Note that \eqref{homocondition} is positively homogeneous in $w$.
Therefore we need to give an additional constraint on $w$.
Following \cite{Lanckriet2002} and \cite{Huang2004} we may set $w^T(\mu_P-\mu_N)=1$ without loss of generality, 
leading to the constrained SOCP subproblem.
Following \cite{Cousins2017}, we set $\|w\|=1$, leading to a constrained concave (or convex) optimization
problem which is easier to solve (see \eqref{convexRelax} below).
The problem \eqref{MPMFe} is then transformed into the problem
\begin{align}\label{MPMF}
\begin{split}
&\min_{\alpha_P,\alpha_N,w}\frac{(1-p)\alpha_N+\beta^2p\alpha_P}{1-\alpha_P}\\
\text{s.t.}\;\; &\|w\|=1\\
&\pi(\alpha_N)\sqrt{w^T\Sigma_N w}+\pi(\alpha_P)\sqrt{w^T\Sigma_P w}=w^T\mu_{pn} 
\end{split}
\end{align}
We apply the alternative descent method to solve the non-convex problem \eqref{MPMF}.
Details are presented in Algorithm \ref{alg-1}.
We initialize the classifier $w_1=({\mu_P-\mu_N})/{\|\mu_P-\mu_N\|}$, so $\|w_1\|=1$.
Assume that at the $t$-th round, we have obtained the classifier $w_t$ with $\|w_t\|=1.$
Let $A_t=\sqrt{w_t^T\Sigma_P w_t}$, $B_t=\sqrt{w_t^T\Sigma_N w_t}$ and $C_t=w_t^T(\mu_P-\mu_N)$.
Then we seek $\alpha_{P,t}$ and $\alpha_{N,t}$ minimizing $Q_F(\alpha_P,\alpha_N;w_t)$
with the fixed classifier $w_t$, that is,
\begin{align}\label{sub-1}
\begin{split}
\alpha_{P,t},\alpha_{N,t}= &\arg\min_{\alpha_P,\alpha_N}\frac{(1-p)\alpha_N+\beta^2p\alpha_P}{(1-\alpha_P)}\\
&\text{s.t.}\;\;\pi(\alpha_N)B_t+\pi(\alpha_P)A_t=C_t 
\end{split}
\end{align}
From the equality constraint in \eqref{sub-1} it follows that
\be\label{alphaN}
\alpha_N=\frac{{B_t}^2}{{B_t}^2+\left[C_t-\pi(\alpha_P)A_t\right]^2}.
\en
Substituting $\alpha_N$ into \eqref{sub-1} gives 
\begin{align}\label{solution-1}
\begin{split}
&\alpha_{P,t}\\
&=\arg\min_{\alpha_P}\frac{(1-p)B_t^2}{[{B_t}^2+(C_t-\pi(\alpha_P)A_t)^2](1-\alpha_P)}
+\frac{\beta^2 p}{\pi^2(\alpha_P)}
\end{split}
\end{align}
From the equality constraint in \eqref{sub-1} again we have that $C_t-\pi(\alpha_P)A_t\ge0$,
implying that $\alpha_P\ge{A_t^2}/({A_t^2+C_t^2})$.
Make use of a grid-based search method to find $\alpha_{P,t}\in[{A_t^2}/({A_t^2+C_t^2}),1)$.
The corresponding $\alpha_{N,t}$ can then be computed from \eqref{alphaN}.

We now update the classifier $w_t$. Note that the objective $Q_F(\alpha_{P,t},\alpha_{N,t};w)$ will not
decrease if both $\alpha_{P,t},\alpha_{N,t}$ are fixed. Therefore we fix $\alpha_{P,t}$ and update
both $\alpha_{N,t}$ and the classifier simultaneously, which gives a smaller misclassification
probability ($\alpha_{N,t}'<\alpha_{N,t}$).
This can be achieved by solving the problem \eqref{MPMF} at $\alpha_P=\alpha_{P,t}$.
Let $\tau_t=\pi(\alpha_{P,t})$. The problem \eqref{MPMF} can then be transformed into the problem
\begin{align*}
\begin{split}
&\alpha_{N,t}'=\arg\min_{\alpha_N, w}\alpha_N\\
\text{s.t.}\;\;\;&\|w\|=1 \\
&\pi(\alpha_N)\sqrt{w^T\Sigma_N w}+\tau_t\sqrt{w^T\Sigma_P w}=w^T\mu_{pn}.
\end{split}
\end{align*}
Note that $\pi(\alpha)$ is monotonically decreasing with $\alpha$, so the above optimization problem is 
equivalent to the fractional programming (FP) problem:
\begin{align}\label{sub-2}
\begin{split}
&\max_{w}\frac{w^T\mu_{pn}-\tau_t\sqrt{w^T\Sigma_P w}}{\sqrt{w^T\Sigma_N w}}\\
&\;\;\text{s.t.}\;\;\;\|w\|= 1.
\end{split}
\end{align}
Write $\lambda_{t}(w)=[w^T\mu_{pn}-\tau_t\sqrt{w^T\Sigma_P w}]/{\sqrt{w^T\Sigma_N w}}$.
It is not necessary to find an exact solution of \eqref{sub-2}. In fact, it is enough to find an inexact
solution $w_{t+1}$ such that $\lambda_{t}(w_{t+1})>\lambda_{t}(w_{t})$ and $\|w_{t+1}\|=1$.
The corresponding misclassification probability is then given as
$\alpha_{N,t}'={1}/[{1+\lambda_t^2(w_{t+1})}]$. Let $\eta_{t}=\lambda_{t}(w_{t})>0$ and let us
define $f:\R^p\to\R$ by
\ben
f_{t}(w):=w^T\mu_{pn}-\tau_t\sqrt{w^T\Sigma_P w}-\eta_{t}\sqrt{w^T\Sigma_N w}.
\enn
Then $f_t$ has the following properties.

\begin{lemma}\label{f}
1) $f_t(w)$ is a concave function; 2) The condition $f_t(w)>0$ is positively homogeneous in $w$;
3) $f_t(w_t)=0$; 4) If $\hat w$ satisfies that $f_t(\hat w)>0$, then we have $\lambda_{t}(\hat w)>\eta_t$.
\end{lemma}

\begin{proof}
1) Since $f_{t}(w)=w^T\mu_{pn}-\tau_t\|\Sigma_P^{1/2}w\|_2-\eta_{t}\|\Sigma_N^{1/2}w\|_2$
and $\|\Sigma_P^{1/2}w\|_2,\|\Sigma_N^{1/2}w\|_2$ are convex function,
then it follows that $f(w)$ is concave.

2) Suppose there is a $\hat w\in\R^p$ such that $f_t(\hat w)>0$.
Then, for any $s>0$ we have $f_t(s\hat w)=sf_t(\hat w)>0$, that is,
the condition $f_t(w)>0$ is positively homogeneous in $w$.

3) Since $\eta_{t}=\lambda_{t}(w_{t})$, it is easy to see that $f_t(w_t)=0$.

4) The condition $f_t(\hat w)>0$ implies that
$\hat w^T\mu_{pn}-\tau_t\sqrt{\hat w^T\Sigma_P\hat w}-\eta_{t}\sqrt{\hat w^T\Sigma_N\hat w}>0$.
Then, by the definition of $\lambda_t(w)$ and $\eta_t$ we have $\lambda_{t}(\hat w)>\eta_t$.
\end{proof}

We now use Lemma \ref{f} to find an inexact solution of \eqref{sub-2}.
If $\nabla f_t(w_t)= 0$, then set $w_{t+1} = w_t$. The algorithm terminates.
Otherwise, $f_t(w)$ does not attain its maximum at $w_t$. Hence there is a $w_{t+1}$ such that
$f_t(w_{t+1})>0$. Choose $w_{t+1}$ as the solution of the optimization problem:
\be\label{convexRelax}
\max_{w} f_t(w)\qquad \text{s.t.}\;\;\;\|w\|\le 1.
\en
We may use a gradient projection method to solve the problem \eqref{convexRelax}.
Take the initial value $v_1=w_{t}$ and the step size $\gamma_k= 1/k,\;k=1,2,\cdots$.
At the $k$-th sub-step, $v_k$ is updated as follows:
\ben
u_k=v_k+\gamma_k\nabla f_{t}(v_k),\;\;\; v_k=\frac{u_k}{\|u_k\|},
\enn
where
\ben
\nabla f_{t}(v_k)=(\mu_P-\mu_N)-\frac{\tau_t\Sigma_P v_k}{\sqrt{v_k^T\Sigma_P v_k}}
-\frac{\eta_t\Sigma_N v_k}{\sqrt{v_k^T\Sigma_N v_k}}.
\enn
The algorithm terminates at the $k$-th step if $f_t(v_k)$ and $f_t(v_{k+1})$ satisfy the conditions
\begin{align}\label{terminated-condition}
\begin{split}
&f_t(v_{k+1}) \ge 0,\;\;\;f_t(v_k) \ge 0,\\
&\|\nabla f_t(v_{k+1})\| \le 0.0001.
\end{split}
\end{align}
Then we update the classifier with $w_{t+1}=v_k$. After running $T$-rounds, $b$ is given as
\be\label{b}
b_T=w_T^T\mu_P-\pi(\alpha_{P,T})\sqrt{w_T^T\Sigma_P w_T}.
\en

\begin{theorem}\label{the:covmpmf}
After running Algorithm \ref{alg-1}, the objective value $Q(\alpha_{P,t},\alpha_{N,t};w_t)$ converges.
\end{theorem}

\begin{proof}
Note first that Algorithm \ref{alg-1} uses the alternative descent method to solve the MPMF problem \eqref{MPMF}.
At the $t$-th round, we obtained the misclassification probability $\alpha_{P,t},\alpha_{N,t}$
by minimizing $Q_F(\alpha_{P},\alpha_{N};w_t)$. We then fix $\alpha_{P,t}$ and seek a classifier $w_{t+1}$
which makes $\alpha_{N,t}'$ smaller. By the definition of $Q_F$ we have
$Q_F(\alpha_{P,t},\alpha_{N,t};w_t)>Q_F(\alpha_{P,t},\alpha_{N,t}'; w_{t+1})$.
Repeat this process. Then, at the $(t+1)$-th round, update
$(\alpha_{P,t+1},\alpha_{N,t+1})=\arg\min Q_F(\alpha_{P},\alpha_{N};w_{t+1})$.
Thus, $Q_F(\alpha_{P,t},\alpha_{N,t};w_t)>Q_F(\alpha_{P,t+1},\alpha_{N,t+1};w_{t+1})$.
This means that the objective $Q_F(\alpha_{P,t},\alpha_{N,t};w_t)$ is monotonically decreasing with $t$,
implying that $Q_F(\alpha_{P,t},\alpha_{N,t};w_t)$ converges as $t\to\infty$.
The proof is thus complete.
\end{proof}

Suppose $(\alpha_P,\alpha_N)$ converges to $(\alpha_P^\star,\alpha_N^\star)$. Then the classifier $w$ also 
converges to the optimal solution of \eqref{convexRelax} in which $\tau_t = \pi(\alpha_P^\star)$ and 
$\eta_t=\pi(\alpha_N^\star)$.
In the experiments, we will see that both $\alpha_P,\alpha_N$ are convergent for several real datasets 
(see Fig.\ref{fig:alpha}).

\begin{algorithm}[tb]
\caption{MPM for the $F_\beta$-measure (MPMF)}\label{alg-1}
\begin{algorithmic}[1]
   \STATE {\bfseries Input:} Mean and covariance matrix of positive and negative samples are 
          $(\mu_P,\Sigma_P)$ and $(\mu_N,\Sigma_N)$, respectively. 
          Initial classifier $w_1=({\mu_P-\mu_N})/{\|\mu_P-\mu_N\|}$.
   \FOR{$t=1$ {\bfseries to}}
   \STATE Given $w_t$, find $\alpha_{P,t}$ by solving problem \eqref{solution-1} with a grid-based search method. 
   The corresponding $\alpha_{N,t}$ is calculated from \eqref{alphaN}.
   \STATE Calculate the objective function \\ $Q = Q_F(\alpha_{P,t},\alpha_{N,t};w_{t})$.
   \STATE Computer $\tau_t=\pi(\alpha_{P,t})$ and $\eta_t=\lambda_t(w_t)$.
   \STATE Set $v_1=w_t$ and solve \eqref{convexRelax}.
   \FOR   {$k =1,2,\cdots,$}
   \STATE $u_k=v_k+\gamma_k\nabla f_{t}(v_k)$
   \STATE $v_{k+1}=\frac{u_k}{\|u_k\|}$
   \STATE Terminate if the conditions \eqref{terminated-condition} are satisfied.
   \ENDFOR
   \STATE Set $w_{t+1} = v_{k+1}$ and the corresponding $\alpha_{N,t}'$ is given as $\alpha_{N,t}' = {1}/[{1+\lambda_t^2(w_{t+1})}]$.
   \STATE Calculate the objective function \\ $Q' = Q_F(\alpha_{P,t},\alpha_{N,t}';w_{t+1})$.
   \STATE If $Q-Q'\le 0.0001$, terminate.
   \ENDFOR
   \STATE Compute $b$ by using \eqref{b}.
\end{algorithmic}
\end{algorithm}

We have the following remarks on our Algorithm \ref{alg-1}.

1) For imbalanced data, the mean of the positive and negative samples is different, 
that is, $\mu_P\neq\mu_N$. Thus, when $\|\mu_P-\mu_N\|$ becomes larger, $\alpha_N$ and $\alpha_P$ 
will get smaller.

2) Given $w_t$, we have $C_t\ge\pi(\alpha_P)A_t$. Then
\ben
\alpha_P\ge\frac{1}{1+\left(\frac{C_t}{A_t}\right)^2}\ge\frac{1}{1+\frac{d^2}{w_t^T \Sigma_P w_t}}
\ge\frac{1}{1+\frac{d^2}{\lambda_{min}^{P}}},
\enn
where $\lambda_{min}^{P}$ is the minimal eigenvalue of $\Sigma_P$.
Similarly, 
\ben
\alpha_N\ge\frac{1}{1+\left(\frac{C_t}{B_t}\right)^2}\ge\frac{1}{1+\frac{d^2}{w_t^T \Sigma_N w_t}}
\ge\frac{1}{1+\frac{d^2}{\lambda_{min}^{N}}},
\enn
where $\lambda_{min}^{N}$ is the minimal eigenvalue of $\Sigma_N$.

3) We have the implicit constraint $w^T\mu_{pn}>0$ for the classifier $w$. This is also satisfied by
the initial classifier $w_1=({\mu_P-\mu_N})/{\|\mu_P-\mu_N\|}$.

4) Suppose $\mu_P,\mu_N,\Sigma_P,\Sigma_N$ are fixed. Then the MPMF model depends on
$\tau(\beta,p)={\beta^2p}/({1-p})$.

5) Our MPMF algorithm for $F_\beta$ can be naturally extended to other non-decomposable performance
measures listed in Table \ref{NDMeasures}. In fact, by replacing the objective function $Q_{F}$ with
the general objective function $Q_{nd}(\alpha_P,\alpha_N)$ in the problem \eqref{MPMFmin},
we have an MPM model with a non-decomposable performance measure (MPMND), leading to
the following optimization problem similar to \eqref{MPMF}:
\begin{align}\label{MPMND}
\begin{split}
&\min_{\alpha_P,\alpha_N, w} \sum_{i,j=0}^ \infty \tau_{ij} \alpha_N^i \alpha_P^j \\
&\;\text{s.t.}\;\;\;\|w\|=1  \\
& \pi(\alpha_N)\sqrt{w^T\Sigma_N w}+\pi(\alpha_P)\sqrt{w^T\Sigma_P w}= w^T\mu_{pn}.
\end{split}
\end{align}
This optimization problem can also be solved by the alternative descent method, leading to an algorithm
similar to Algorithm \ref{alg-1}. The only difference is step 3, where the problem \eqref{solution-1}
is replaced with the problem
\ben
\alpha_{P,t}=\arg\min Q_{nd}\left(\alpha_P,\frac{{B_t}^2}{{B_t}^2+[C_t-\pi(\alpha_P)A_t]^2};w_t\right)
\enn

\begin{table}
\caption{$F_\beta$ measure under different combinations of $\beta$ and the proportion $p$ of
positive samples. In this experiment, the prior information is explicitly given.
}\label{Experiment-1}
\begin{center}
\begin{small}
\begin{sc}
\begin{tabular}{l|ll|ll|}
\toprule
\multirow{2}{*}{p}  & \multicolumn{2}{c|}{$\beta =1 $} & \multicolumn{2}{c|}{$\beta =3$}\\  \cline{2-5}
 & $\alpha_P(\mbox{FNR})$ & $\alpha_N(\mbox{FPR})$ & $\alpha_P(\mbox{FNR})$ & $\alpha_N(\mbox{FPR})$ \\
\midrule
0.5   & 0.1646 & 0.1995 & 0.0846 & 0.5447  \\
0.4   & 0.1847 & 0.1762 & 0.0946 & 0.4436  \\
0.3   & 0.2047 & 0.1592 & 0.1146 & 0.3224  \\
0.2   & 0.2347 & 0.1406 & 0.1246 & 0.2847  \\
0.1   & 0.2847 & 0.1203 & 0.1646 & 0.1995  \\
0.05  & 0.3247 & 0.1093 & 0.1947 & 0.1671  \\
0.01  & 0.3747 & 0.0992 & 0.2947 & 0.1172  \\
\bottomrule
\end{tabular}
\end{sc}
\end{small}
\end{center}
\end{table}

\begin{table}
\caption{Details of the training imbalanced datasets with $p$ the proportion of the positive samples.
}\label{dataset}
\begin{center}
\begin{small}
\begin{sc}
\begin{tabular}{lllll}
\toprule
dataset    & training   & testing   & feature &  p  \\
\midrule
letter     & 15000      & 5000      & 16      & 0.0384 \\
breast     & 462        & 219       & 10      & 0.3499 \\
segment    & 1299       & 1009      & 19      & 0.1428 \\
usps       & 7291       & 2007      & 256     & 0.1000 \\
ijcnn      & 49990      & 91701     & 22      & 0.0957 \\
covtype    & 500000     & 81012     & 54      & 0.3646 \\
skin       & 200000     & 45057     & 3       & 0.2075 \\
sensorless & 39999      & 18508     & 48      & 0.0909 \\
\bottomrule
\end{tabular}
\end{sc}
\end{small}
\end{center}
\vskip -0.1in
\end{table}

\section{Kernelization}\label{sec5}

Due to the non-decomposability of the $F_\beta$-measure, most algorithms for the $F_\beta$-measure maximization 
focus on linear classifiers which may not be always effective for some real-world problems.
In this paper, we apply the kernel trick to the minimax probability machine for the $F_\beta$-measure
to derive the Kernel Minimax Probability Machine for $F_\beta$ (KMPMF) which yields a nonlinear classifier. 

By \cite{Lanckriet2002}, the kernel trick works in the MPM model.
Let $K:\R^d\times\R^d\to\R$ be a positive kernel function.
Given the positive samples $\{x_P^i\}_{i=1}^{N_P}$ and the negative samples $\{x_N^i\}_{i=1}^{N_N}$,
define the Gram matrix $\bf{K}=(\bf{K}_{ij})$ with
\begin{align*}
\bf{K}_{ij} =
\begin{cases}
K(x_P^i,x_P^j)\;\;\text{if}\; i\le N_P,\;j\le N_P,\\
K(x_P^i,x_N^{j-N_P})\;\;\text{if}\; i\le N_P,\;j>N_P,\\
K(x_N^{i-N_P},x_P^j)\;\; \text{if}\; i>N_P,\;j\le N_P,\\
K(x_N^{i-N_P},x_N^{j-N_P})\;\;\text{if}\; i>N_P,\;j>N_P.\\
\end{cases}
\end{align*}
The first $N_P$ rows and the last $N_N$ rows of $\bf K$ are named the positive Gram matrix $\bf K_P$ and
the negative Gram matrix $\bf K_N$, respectively.
Denote by $\bf I_P$ and $\bf I_N$ their corresponding column average, which are both $N_P+N_N$
dimensional vectors. Define
\ben
{\bf L_P}=\frac{1}{\sqrt{N_P}}({\bf K_P-l_P1_{P}^T}),\;\;\;
{\bf L_N }=\frac{1}{\sqrt{N_N}}(\bf{K_N-l_N1_{N}^T}),
\enn
where $\bf 1_{P}$ and $\bf 1_{N}$ are column vectors of ones of dimension $N_P$ and $N_N$, respectively.
We have the following theorem which can be shown similarly as in the proof of Theorem 6 in \cite{Lanckriet2002}.

\begin{theorem}\label{thm2}
If $\bf l_P=\bf l_N$, then the minimal probability decision problem has no solution. Otherwise, 
the optimal decision boundary is determined by the solution $(\alpha_P^\star,\alpha_N^\star,w^\star)$
of the optimization problem
\begin{align*}
\begin{split}
&\min_{\alpha_P,\alpha_N,w}\frac{(1-p)(\alpha_N)+\beta^2p(\alpha_P)}{1-\alpha_P}\\
&\text{s.t.}\;\;\|w\|=1\\
&\pi(\alpha_P)\sqrt{w^T{\bf L_P^T L_P}w}+\pi(\alpha_N)\sqrt{w^T{\bf L_N^T L_N}w}=w^T({\bf l_P-l_N})
\end{split}
\end{align*}
and $b^\star=(w^\star)^T{\bf l_P}-\pi(\alpha_P^\star)\sqrt{(w^\star)^T{\bf L_P^T L_P}w^\star}$.
\end{theorem}

By Theorem \ref{thm2}, a new data point $x$ is predicted by
$\hat y=\mbox{sign}\left[\sum\limits_{i=1}^{N_P}w^\star_{i}K(x_P^i,x)
+\sum\limits_{j=N_P+1}^{N_P+N_N}w^\star_{j}K(x_N^{j-N_P},x)-b^\star\right]$.

\section{Experiments}\label{sec6}

In this section, we present experiments to verify our MPMF model for the $F_\beta$ measure.

We first consider the case when the mean and covariance matrix of the data are given in advance.
Let $\mu_P=(3,1)^T,\mu_N=(-1,-2)^T$,
$\Sigma_P=\begin{bmatrix}
1& 1/2 \\
1/2 & 1
\end{bmatrix}$ and $\Sigma_N =\begin{bmatrix}
1& 1/3 \\
1/3 & 1
\end{bmatrix}$.
Then we calculate the maximal $\alpha_P$ (FNR) and $\alpha_N$ (FPR) with different combinations of $\beta$
and the proportion $p$ of the positive samples. Set $p=0.5,0.4,0.3,0.2,0.1,0.05,0.01$ and $\beta=1,3$.
Table \ref{Experiment-1} presents the FPR and FNR values in the worst-case.
Given the proportion $p$ of the positive samples, a larger $\beta$ helps reduce FNR in the cost of FPR, 
which is very important in imbalanced classification problems.

In the second experiment, we apply our method MPMND to some real-world datasets,
which can be downloaded from LIBSVM website \footnote{https://www.csie.ntu.edu.tw/\textasciitilde cjlin/libsvm/}
and UCI machine learning repository \footnote{http://archive.ics.uci.edu/ml/datasets.php}.
Table \ref{dataset} presents details of these datasets.
We also define $p$ as the proportion of the positive samples, which is the same as in the training and test datasets.
In the case of multi-class datasets, we report results (using one-vs-all classifiers) averaged over the classes.
For each dataset, the reported values of the non-decomposable measures are obtained by averaging over $20$ independent runs.

Let $\{x_P^i\}_{i=1}^{N_P}$ and $\{x_N^i\}_{i=1}^{N_N}$ be the training data points with positive and
negative class labels, respectively. Then the mean and covariance matrix of the dataset can be estimated as
\ben
\hat\mu_P=\frac{\sum_{i=1}^{N_P} x_P^i}{N_P},\;
\hat\Sigma_P=\frac{\sum_{i=1}^{N_P}(x_P^i-\hat\mu_P)(x_P^i-\hat\mu_P)^T}{N_P-1},\\
\hat\mu_N=\frac{\sum_{i=1}^{N_N} x_N^i}{N_N},\;
\hat\Sigma_N=\frac{\sum_{i=1}^{N_N}(x_N^i-\hat\mu_N)(x_N^i-\hat \mu_N)^T}{N_N-1}.
\enn
These estimators will be used to replace the true statistics in the MPMND model.

The plug-in method, which first learns a classification probability function by training a logistical regression 
model on a training dataset and then decides a threshold based on another dataset, is consistent with several 
non-decomposable performance measures.
Table \ref{tab:time} gives the running time of MPMND and the plug-in method on real datasets for the AR, AM, 
and F measures.
Note that the plug-in method learns the same logistical model on the training dataset for all performance 
measures, so we average its running time.
The results show that our MPMND method is faster than the plug-in method.
Table \ref{tab:measure} presents the corresponding reported performance measures which 
show that our method achieves a comparable performance with the state-of-the-art plug-in method 
(see Table \ref{tab:measure}).
In the minimax approach, the constant $b$ can be directly calculated from \eqref{b}.
However, in practice, it is reasonable to adjust the value of $b$ with a validated dataset 
to obtain a better result.

We consider the F1-measure metric. 
Fig. \ref{fig:Q} presents the reported $Q_F$ value for several different imbalanced datasets.
The horizontal axis represents the number of iterations. From Fig. \ref{fig:Q} it is seen that 
the objective $Q_F$ converges quickly, which is consistent with Theorem \ref{the:covmpmf}.
Fig \ref{fig:alpha} shows the reported value of $\alpha_P$ and $\alpha_N$ as a function of iterations.
It is shown that both $\alpha_P$ and $\alpha_N$ also converge.
After fixing $\alpha_P$, we update the classifier by solving a concave optimization problem \eqref{convexRelax}.
As $\alpha_P$ converges, the classifier is also convergent.

We now compare our MPMF method with several state-of-the-art imbalanced learning algorithms,
including the over-sampling methods: Random Over-Sampling (ROS) \cite{He2009}, 
Synthetic Minority Over-Sampling TEchnique (SMOTE) \cite{Chawla2002},
ADAptive SYNthetic sampling (ADASYN) \cite{He2008}), the under-sampling methods: 
Random Under-Sampling (RUS) \cite{He2009}, Cluster Centroids (CC) \cite{Lemaitre2017}, 
and the ensemble learning methods: Balanced Random Forest Classifier (BRF) \cite{Chen2004} 
and RUSBoost \cite{Seiffert2009}. 
These compared algorithms are implemented with the help of an open-source toolbox Imbalanced-learn \cite{Lemaitre2017}
and the package Scikit-learn \cite{Pedregosa2011}.
For the resampling methods, we train a linear SVM on the new datasets.
Noting that the ensemble methods learn a nonlinear classifier, we also train a KMPMF model.
For computational effectiveness, we randomly choose $200$ positive samples and $200$ negative samples 
as the kernel support vectors, respectively (for the dataset Breast, we choose $100$ positive samples 
and $100$ negative samples). Table \ref{tab:running time-2} gives the running time (in seconds) of 
the algorithms used for imbalanced classification problems.
In most cases, MPMF is faster than the compared imbalanced learning algorithms, especially for large-scale datasets.
This means that this new method MPMF scales well to deal with large-scale imbalanced claffication problems.
Table \ref{tab:measure-2} presents the obtained values of the F1-measure.
From the experimental results, it can be seen that our linear MPMF model has a better performance than 
the compared resampling methods and our nonlinear model KMPMF achieves a comparable result for the F1-measure 
with the compared ensemble methods.

\begin{table*}
\caption{The running time (in seconds) of the MPMND method with different metrics and the plug-in method on real datasets. }
\label{tab:time}
\begin{center}
\begin{small}
\begin{sc}
\resizebox{\textwidth}{!}{
\begin{tabular}{llllllllll}
\toprule
dataset     & AR & AM & QM & HM & GM & GTP & JAC/F1 & F2 & Plug-in\\
\midrule
letter       & 0.0222 & 0.1393 & 0.1578 & 0.1952 & 0.2008 & {\bf 0.2702} & 0.1878 & 0.1923 & 0.1419  \\
breast       & 0.0254 & 0.0341 & 0.0290 & 0.0614 & 0.0676 & 0.0980 & 0.0550 & {\bf 0.0981} & 0.0021  \\
segment      & 0.1081 & 0.2321 & 0.2178 & 0.3425 & 0.3489 & {\bf 0.3576} & 0.2984 & 0.3236 & 0.0280  \\
usps         & 0.8704 & 1.3123 & 1.2332 & 1.4355 & 1.4179 & 1.5656 & 1.4405 & 1.4754 & {\bf 24.088}  \\
covtype      & 0.0225 & 0.4994 & 0.2762 & 0.7239 & 1.1183 & 0.4069 & 0.3504 & 0.3756 & {\bf 9.4688}  \\
ijcnn        & 0.0121 & 0.0770 & 0.0541 & 0.0810 & 0.0851 & 0.1186 & 0.1082 & 0.0866 & {\bf 0.3404} \\
skin         & 0.2357 & 0.1867 & 0.1489 & 0.2044 & 0.2330 & 0.3842 & 0.2455 & 0.2717 & {\bf 0.6936}  \\
sensorless   & 0.0668 & 0.4258 & 0.4552 & 1.0971 & 0.9714 & 0.5558 & 0.5073 & 0.5847 & {\bf 7.9259}  \\
\bottomrule
\end{tabular}
}
\end{sc}
\end{small}
\end{center}
\end{table*}

\begin{table*}
\caption{Comparison of the MPMND and Plug-in methods for different performance measure metrics. 
The values presented correspond to the metrics evaluated on the test datasets (a higher value means 
a better performance). For multi-class datasets, the averaged performance over classes is presented.}
\label{tab:measure}
\centering
\subtable[Reported values of the AR, AM, QM, HM and GM metrics obtained with the MPMND and Plug-In methods.]{
\resizebox{\textwidth}{!}{
\begin{tabular}{l|ll|ll|ll|ll|ll|}
\toprule
\multirow{2}{*}{Dataset}  & \multicolumn{2}{c|}{AR} & \multicolumn{2}{c|}{AM} & \multicolumn{2}{c|}{QM} & \multicolumn{2}{c|}{HM} & \multicolumn{2}{c|}{GM} \\  
\cline{2-11}
 & MPMND & PLUG-IN & MPMND & PLUG-IN & MPMND & PLUG-IN & MPMND & PLUG-IN & MPMND & PLUG-IN \\
\midrule
letter & 0.9705 & {\bf 0.9750} &{\bf 0.8994} & 0.8845 & {\bf 0.9878} & 0.9830 & {\bf 0.9007} & 0.8829 & {\bf 0.9025} & 0.8838\\
breast & {\bf 0.9781} & 0.9639 &{\bf 0.9813} & 0.9664 & {\bf 0.9996} & 0.9985 & {\bf 0.9798} & 0.9663 & {\bf 0.9820} & 0.9663\\
segment & 0.9570 &{\bf 0.9684} & 0.9461 &{\bf 0.9532} & 0.9902 & {\bf 0.9939} & 0.9438 & {\bf 0.9528} & 0.9443 & {\bf 0.9526}\\
usps & {\bf 0.9701} & 0.9699 & {\bf 0.9452} & 0.9277 & {\bf 0.9962} & 0.9911 & {\bf 0.9449} & 0.9242 & {\bf 0.9455} & 0.9260\\
covtype & 0.7404 &{\bf 0.7722} & 0.7600 & {\bf 0.7695} & {\bf 0.9465} & 0.9463 & {\bf 0.7689} & 0.7685 & {\bf 0.7694} & 0.7689\\
ijcnn & 0.9103 &{\bf 0.9238} & 0.8288 & {\bf 0.8363} & 0.9699 & {\bf 0.9712} & 0.8271 & {\bf 0.8339} & 0.8280 & {\bf 0.8351}\\
skin & {\bf 0.9437} & 0.9259 & {\bf 0.9650} & 0.9504 & {\bf 0.9976} & 0.9954 & {\bf 0.9633} & 0.9480 & {\bf 0.9645} & 0.9492\\
sensorless & 0.9305 &{\bf 0.9356} & {\bf 0.8969} & 0.8628 &{\bf 0.9753} & 0.9687 &{\bf 0.8771} & 0.8573 &{\bf 0.8841} & 0.8594\\
\bottomrule
\end{tabular}
}
}

\subtable[Reported values of the G-TP/PR, JAC, F1, and F2 metrics obtained with the MPMND and Plug-In methods.]{
\resizebox{\textwidth}{!}{
\begin{tabular}{l|ll|ll|ll|ll|}
\toprule
\multirow{2}{*}{Dataset}  & \multicolumn{2}{c|}{G-TP/PR} & \multicolumn{2}{c|}{JAC} & \multicolumn{2}{c|}{F1} & \multicolumn{2}{c|}{F2} \\  \cline{2-9}
 & MPMND & PLUG-IN & MPMND & PLUG-IN & MPMND & PLUG-IN & MPMND & PLUG-IN \\
\midrule
letter     & 0.5487 & { \bf 0.6230} & 0.3925 & { \bf 0.4680} & 0.5361 & { \bf 0.6122} & 0.6473 & { \bf 0.6628} \\
breast     & { \bf 0.9678} & 0.9471 & { \bf 0.9374} & 0.8994 & { \bf 0.9676} & 0.9467 & { \bf 0.9814} & 0.9679 \\
segment    & 0.8668 & { \bf 0.8898} & 0.7837 & { \bf 0.8240} & 0.8516 & { \bf 0.8899} & 0.9074 & { \bf 0.9162} \\
usps       & { \bf 0.8592} & 0.8403 & { \bf 0.7502} & 0.7292 & { \bf 0.8533} & 0.8396 & { \bf 0.8795} & 0.8565 \\
covtype    & 0.7026 & { \bf 0.7201} & 0.5468 & { \bf 0.5505} & { \bf 0.7105} & 0.7091 & { \bf 0.8161} & 0.8156 \\
ijcnn      & 0.5365 & { \bf 0.5967} & 0.3816 & { \bf 0.4252} & 0.5524 & { \bf 0.5967} & 0.6067 & { \bf 0.6735} \\
skin       & { \bf 0.8903} & 0.8544 & { \bf 0.7920} & 0.7305 & { \bf 0.8839} & 0.8442 & { \bf 0.9502} & 0.9303 \\
sensorless & { \bf 0.6575} & 0.6269 & { \bf 0.4778} & 0.4687 & { \bf 0.6009} & 0.5860 & { \bf 0.7490} & 0.7065 \\
\bottomrule
\end{tabular}
}
}
\end{table*}

\begin{figure*}[htbp]
\centering
\subfigure[letter]{\includegraphics[width=4cm]{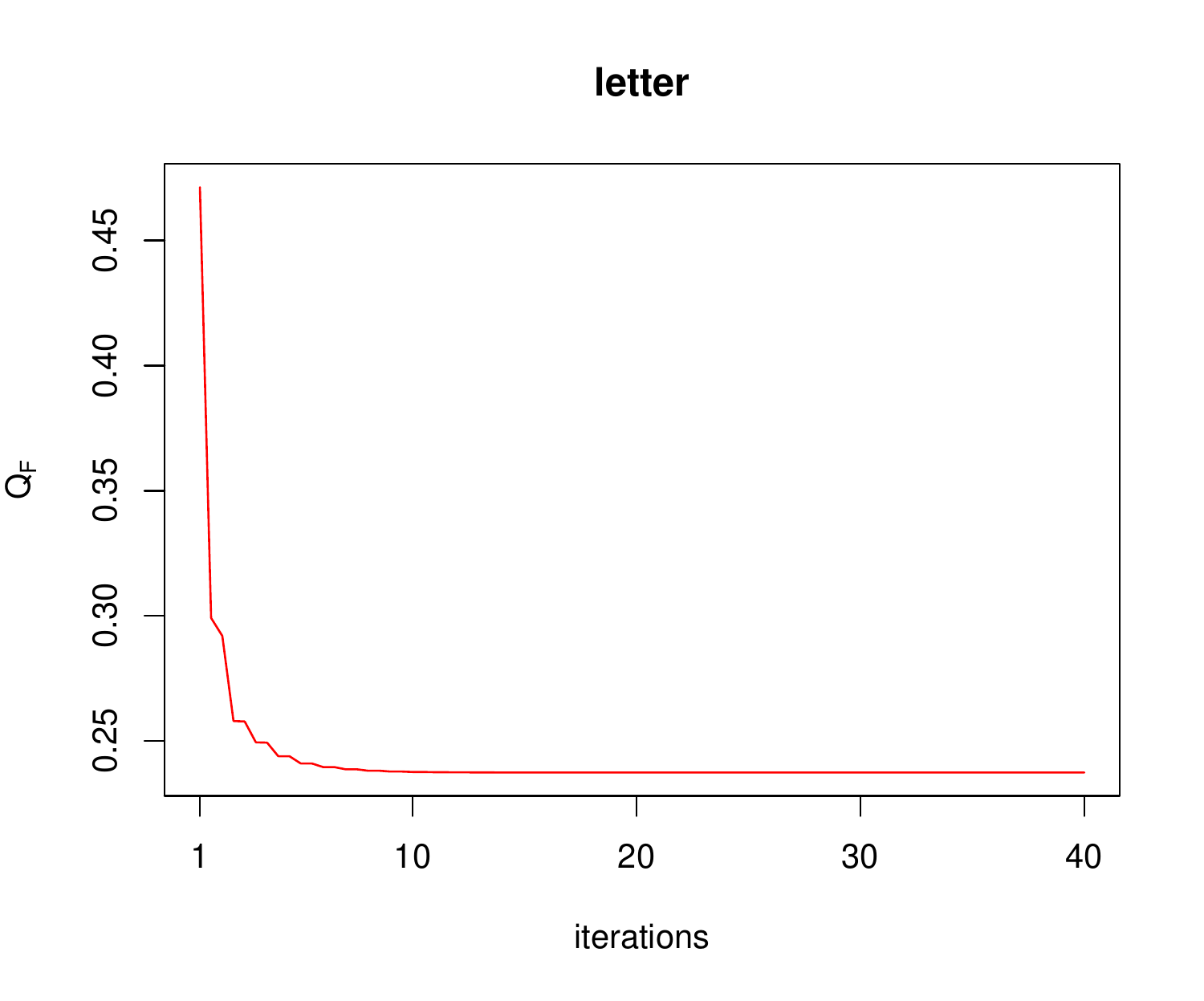}
}
\quad
\subfigure[breast]{\includegraphics[width=4cm]{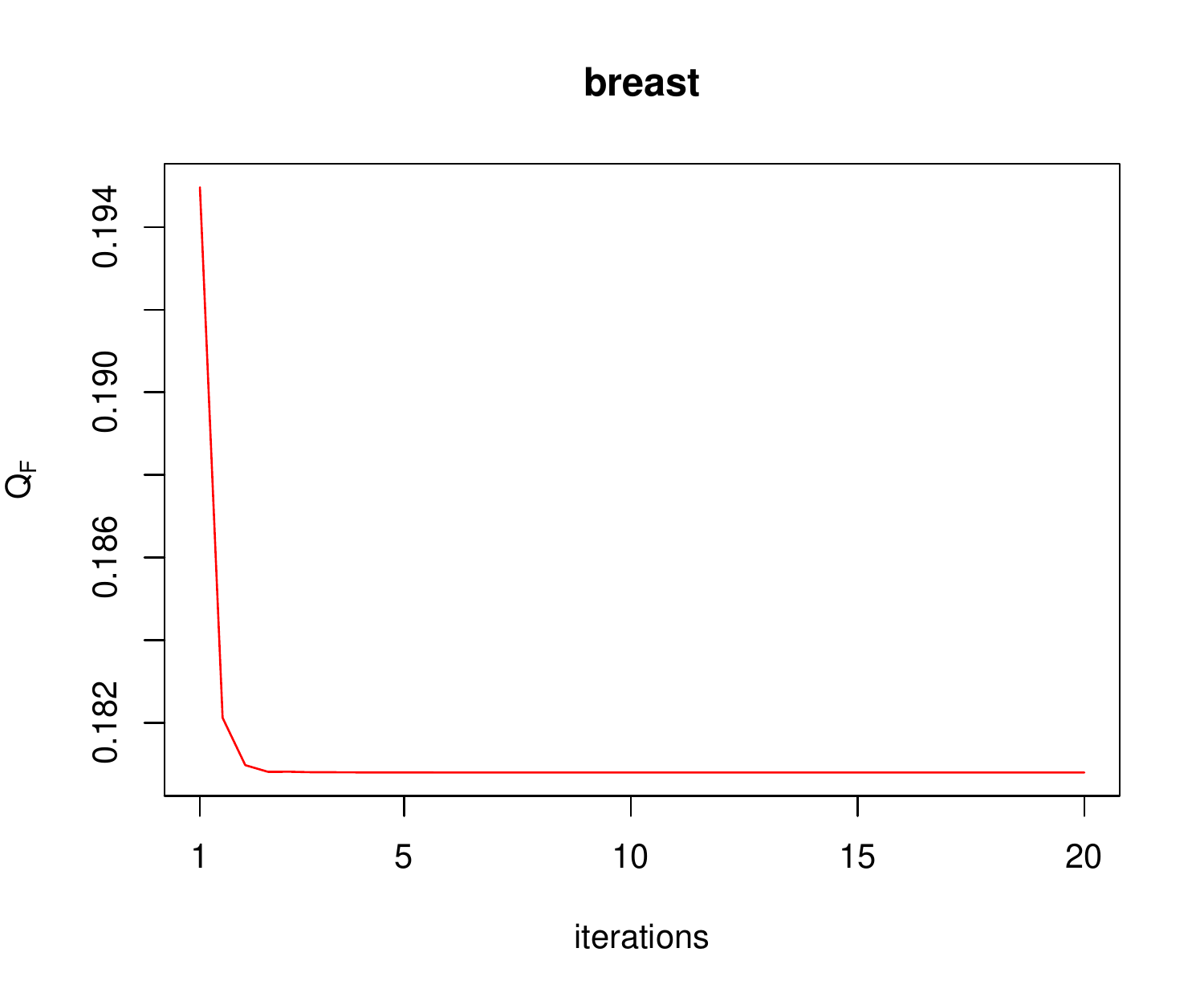}
}
\quad
\subfigure[segment]{\includegraphics[width=4cm]{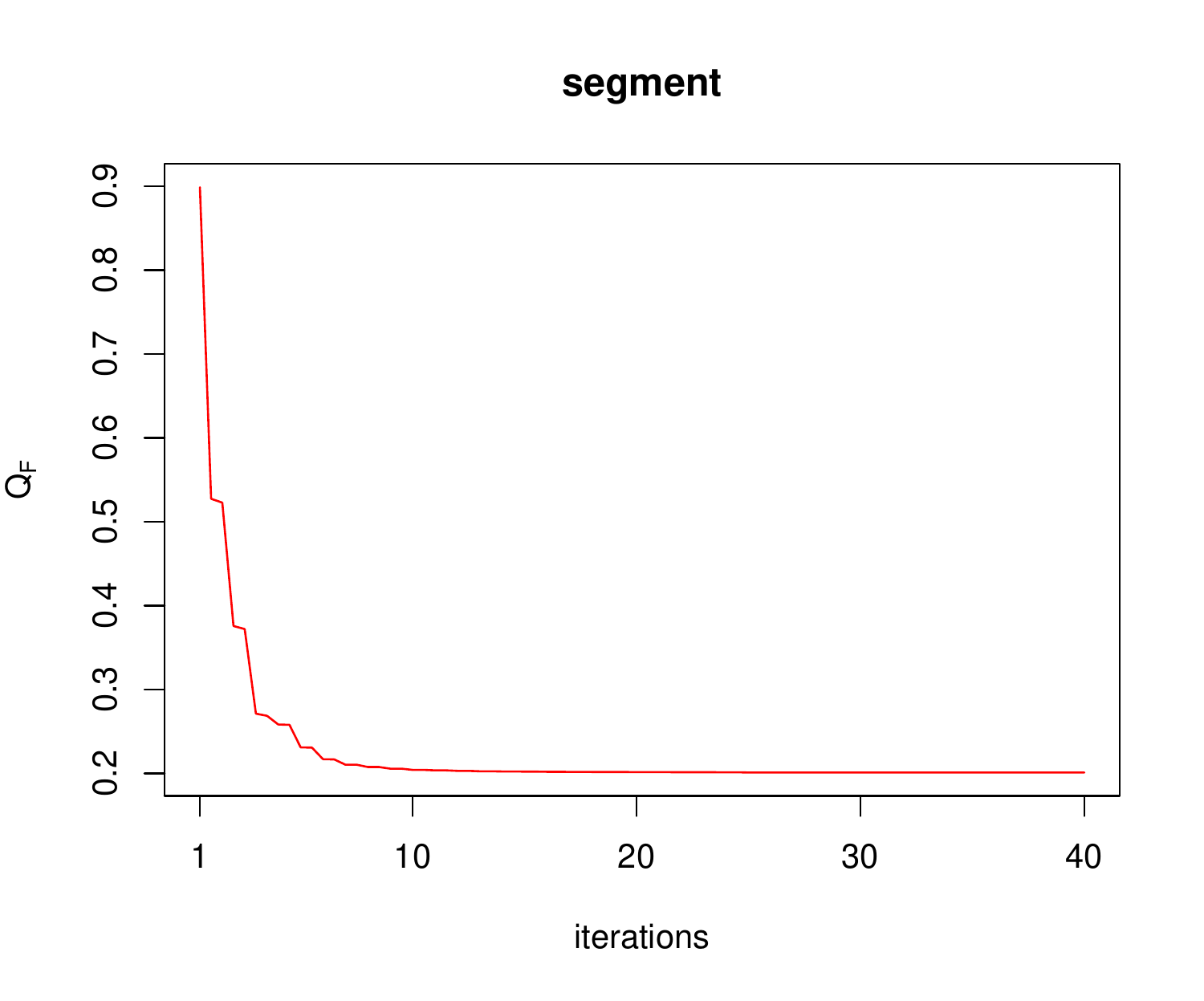}
}
\quad
\subfigure[usps]{\includegraphics[width=4cm]{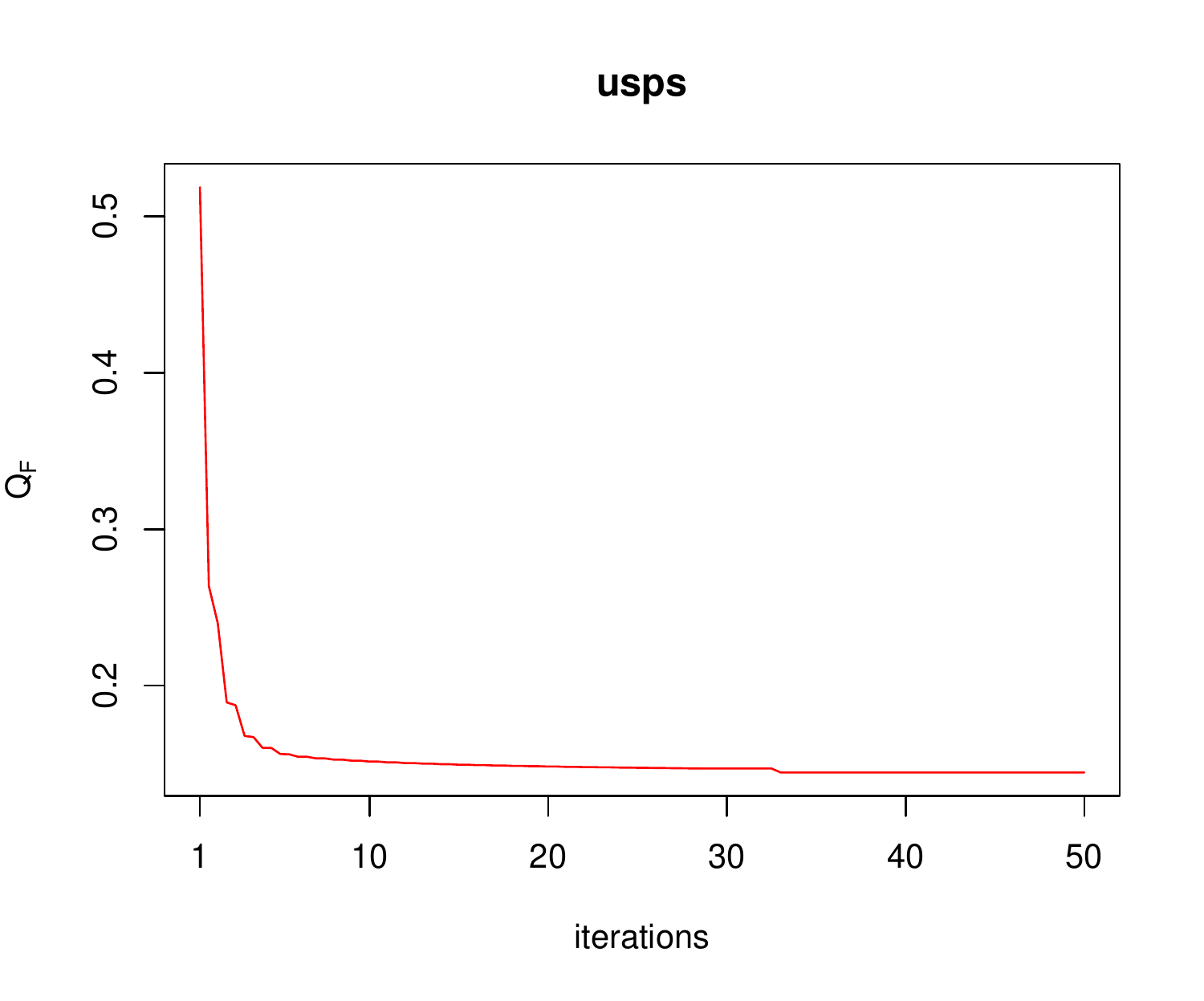}
}
\quad
\subfigure[covtype]{\includegraphics[width=4cm]{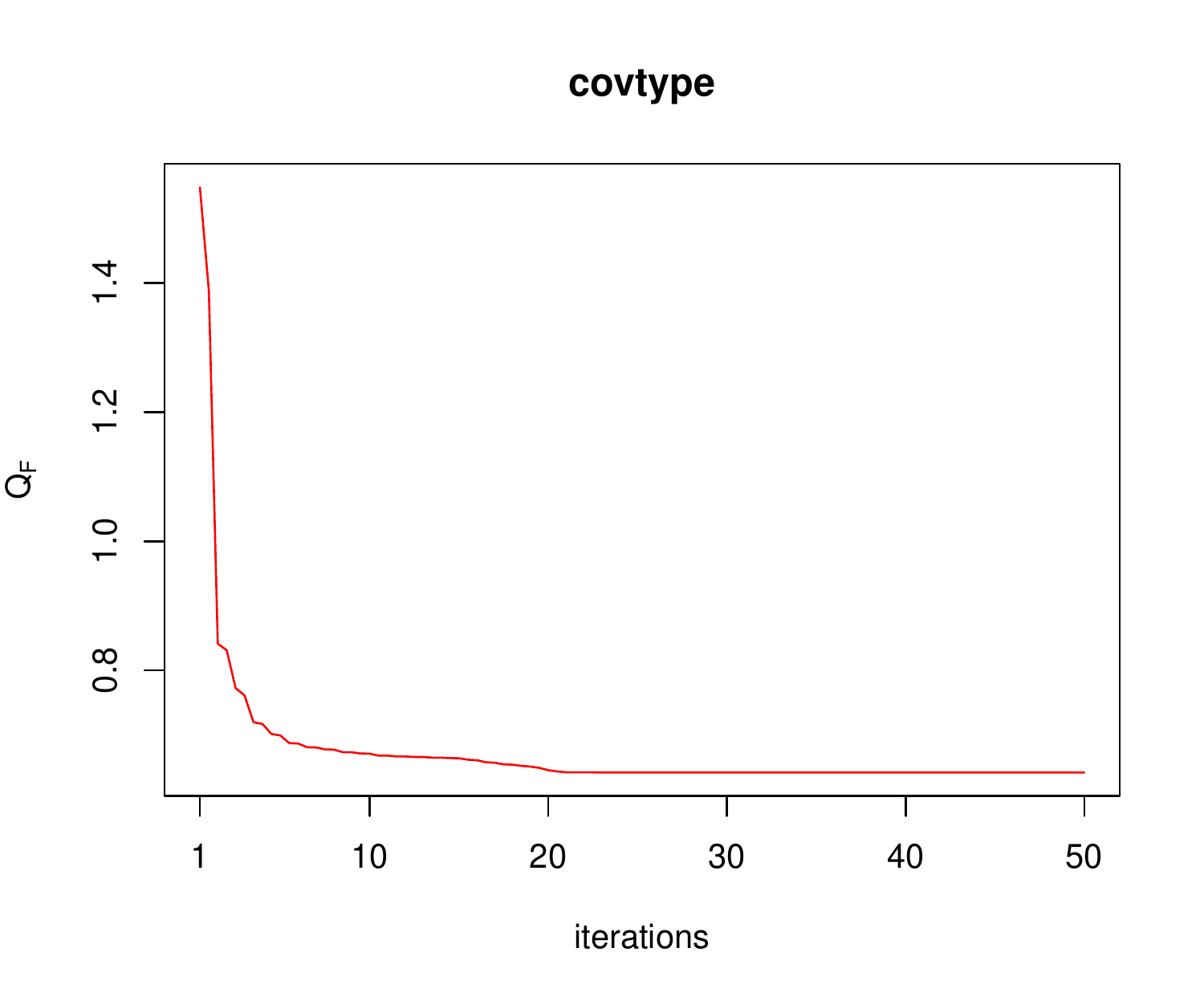}
}
\quad
\subfigure[ijcnn]{\includegraphics[width=4cm]{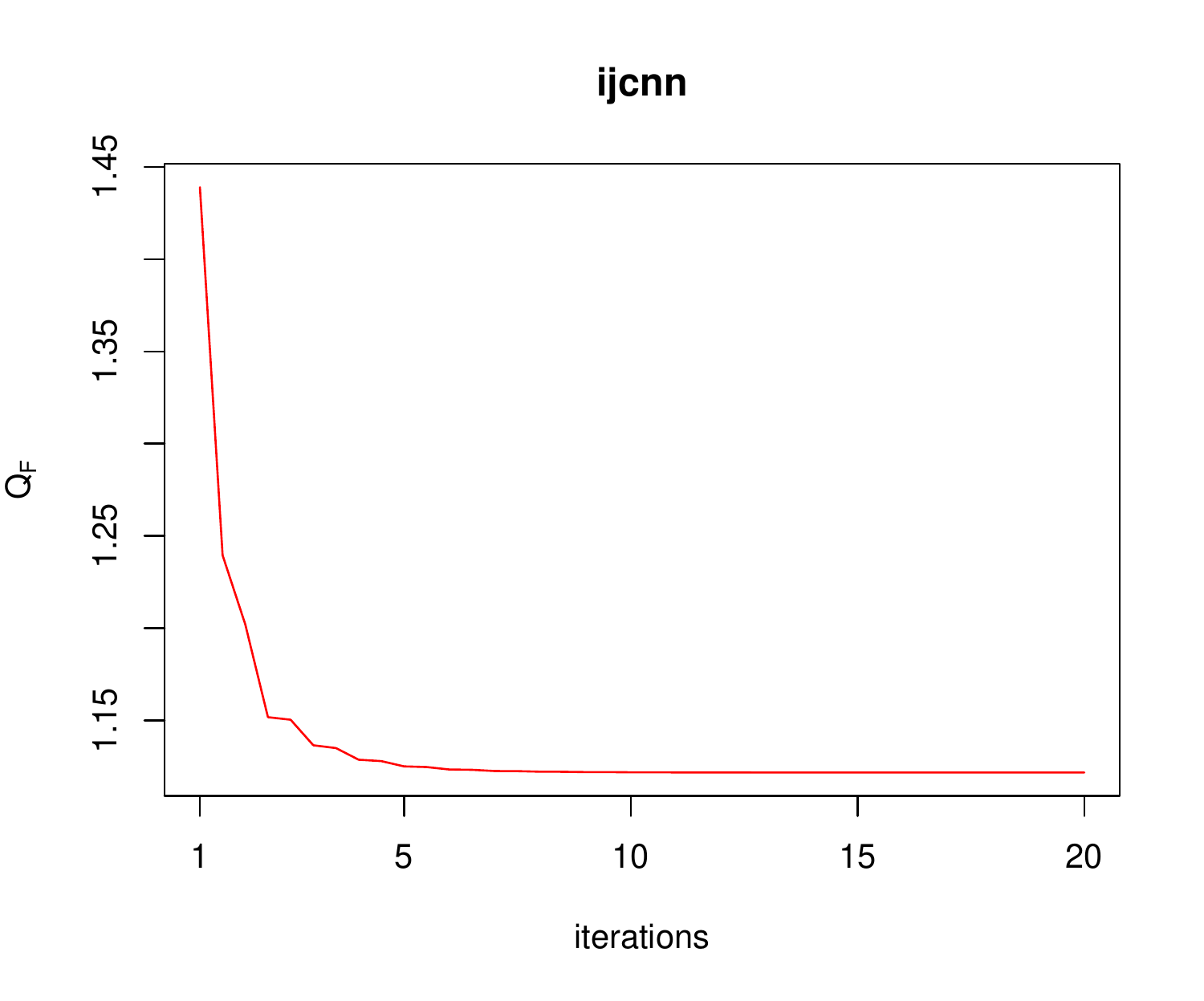}
}
\quad
\subfigure[skin]{\includegraphics[width=4cm]{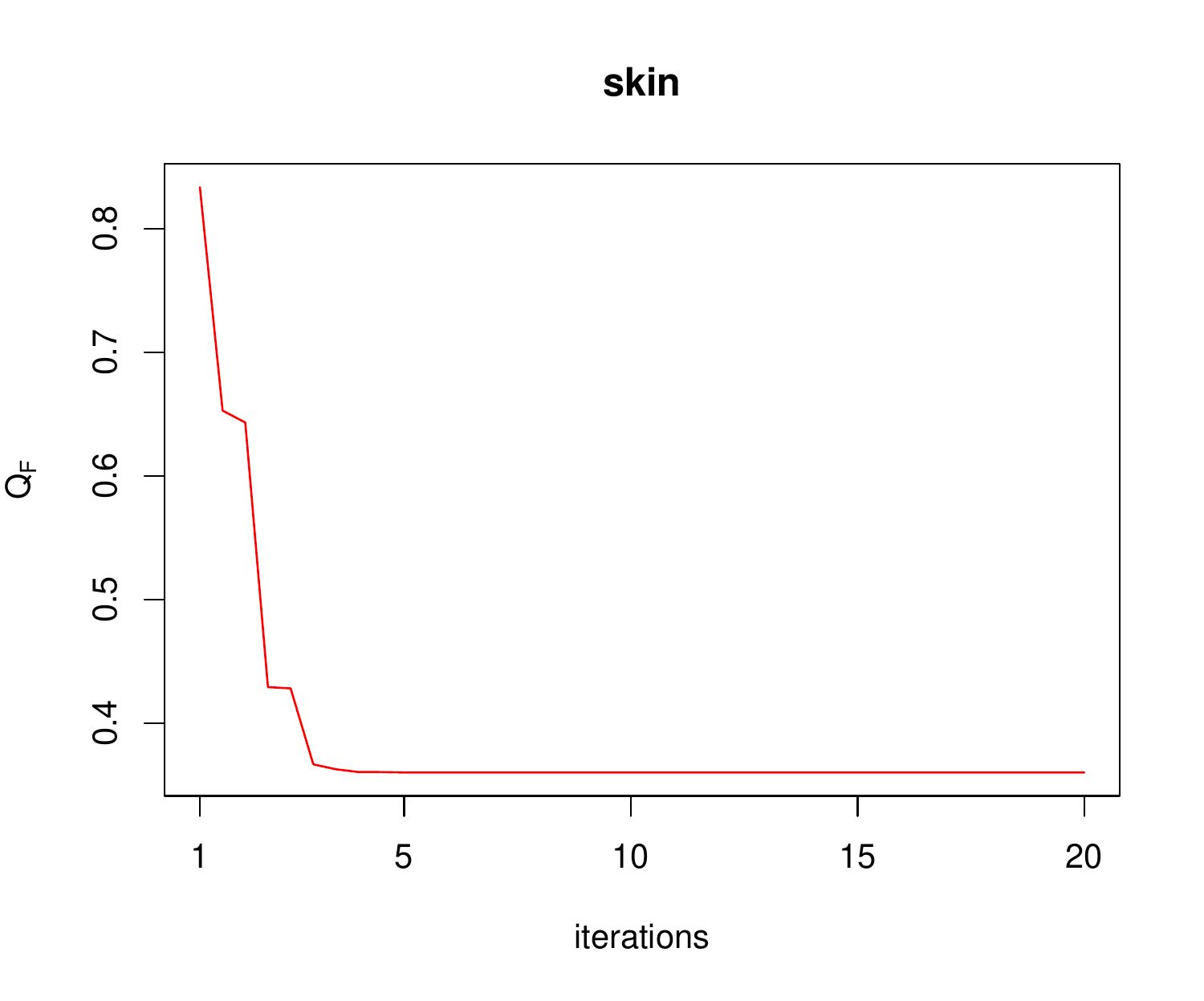}
}
\quad
\subfigure[sensorless]{\includegraphics[width=4cm]{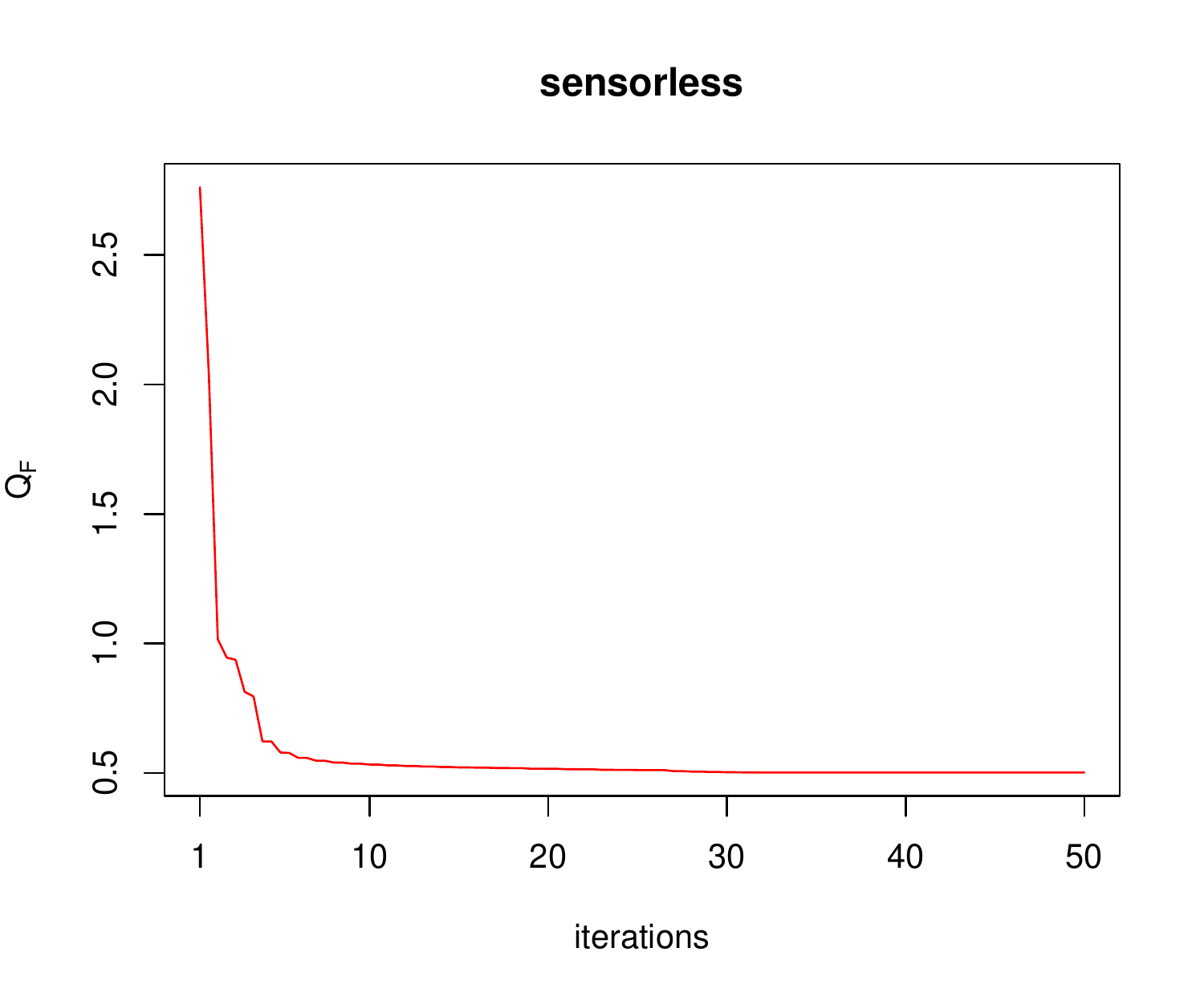}
}
\caption{The reported $Q_F$ as a function of iterations. As the number of iterations increases, 
$Q_F$ monotonically decreases and converges.
}\label{fig:Q}
\end{figure*}

\begin{figure*}[htbp]
\centering
\subfigure[letter]{\includegraphics[width=4cm]{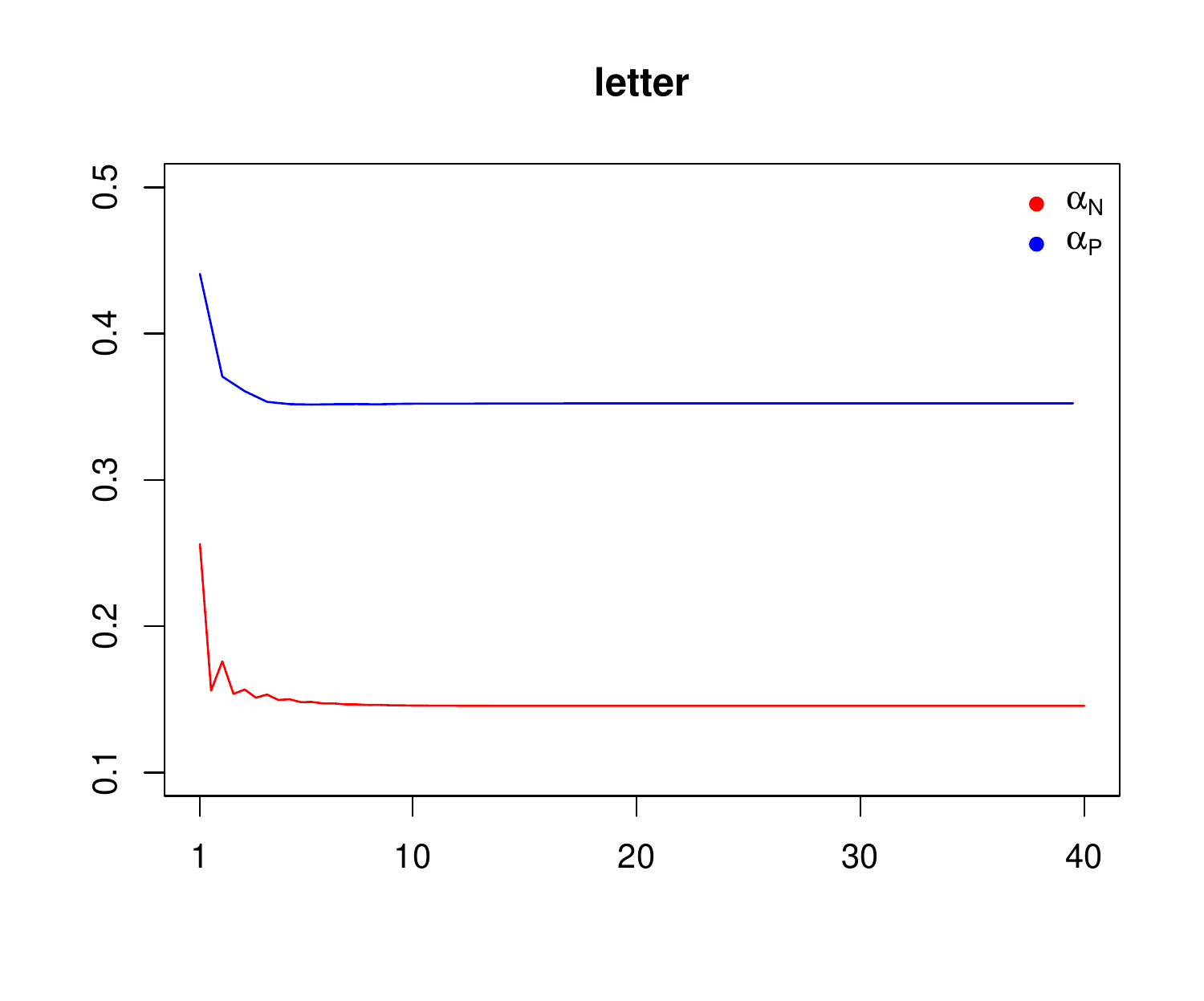}
}
\quad
\subfigure[breast]{\includegraphics[width=4cm]{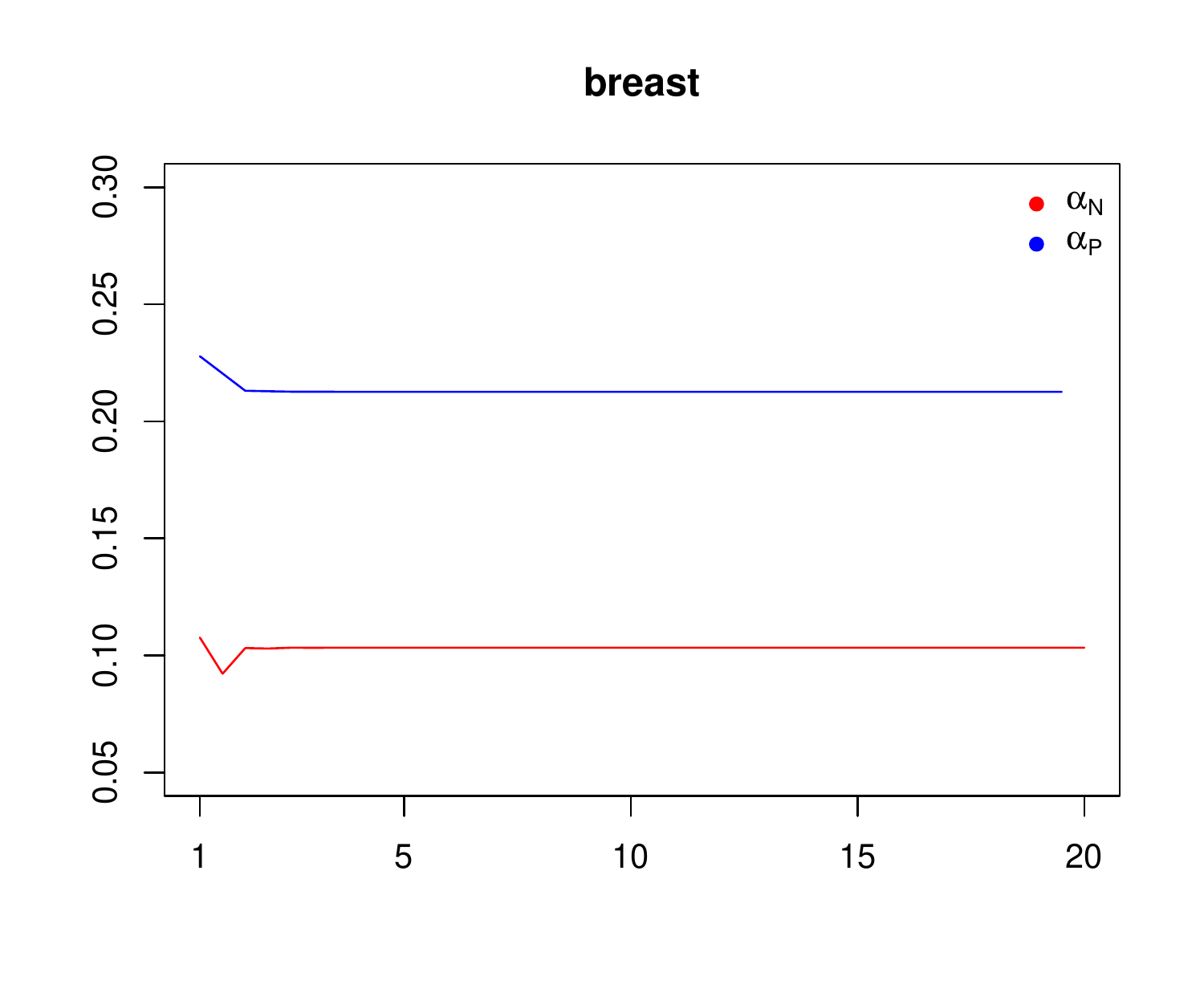}
}
\quad
\subfigure[segment]{\includegraphics[width=4cm]{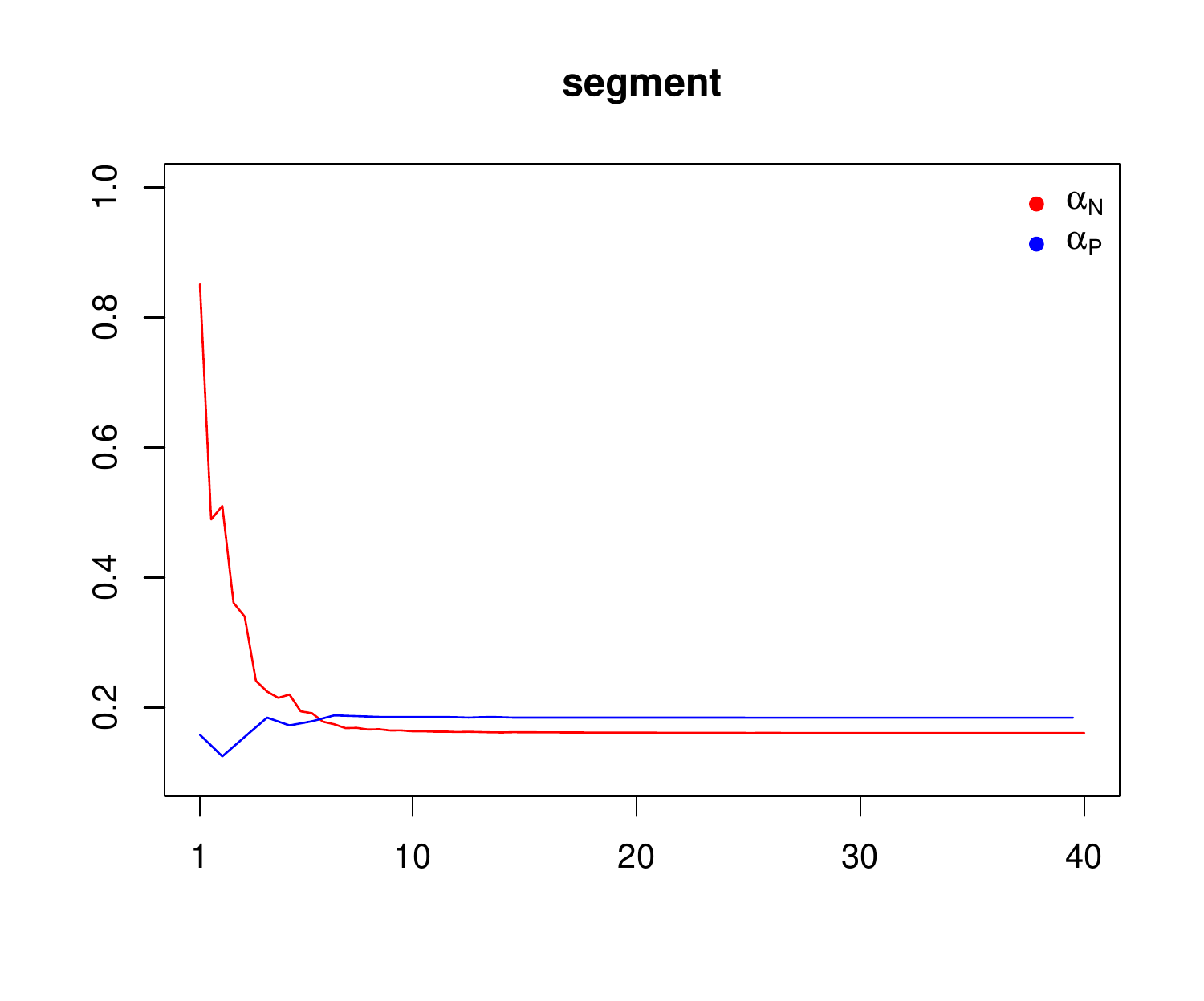}
}
\quad
\subfigure[usps]{\includegraphics[width=4cm]{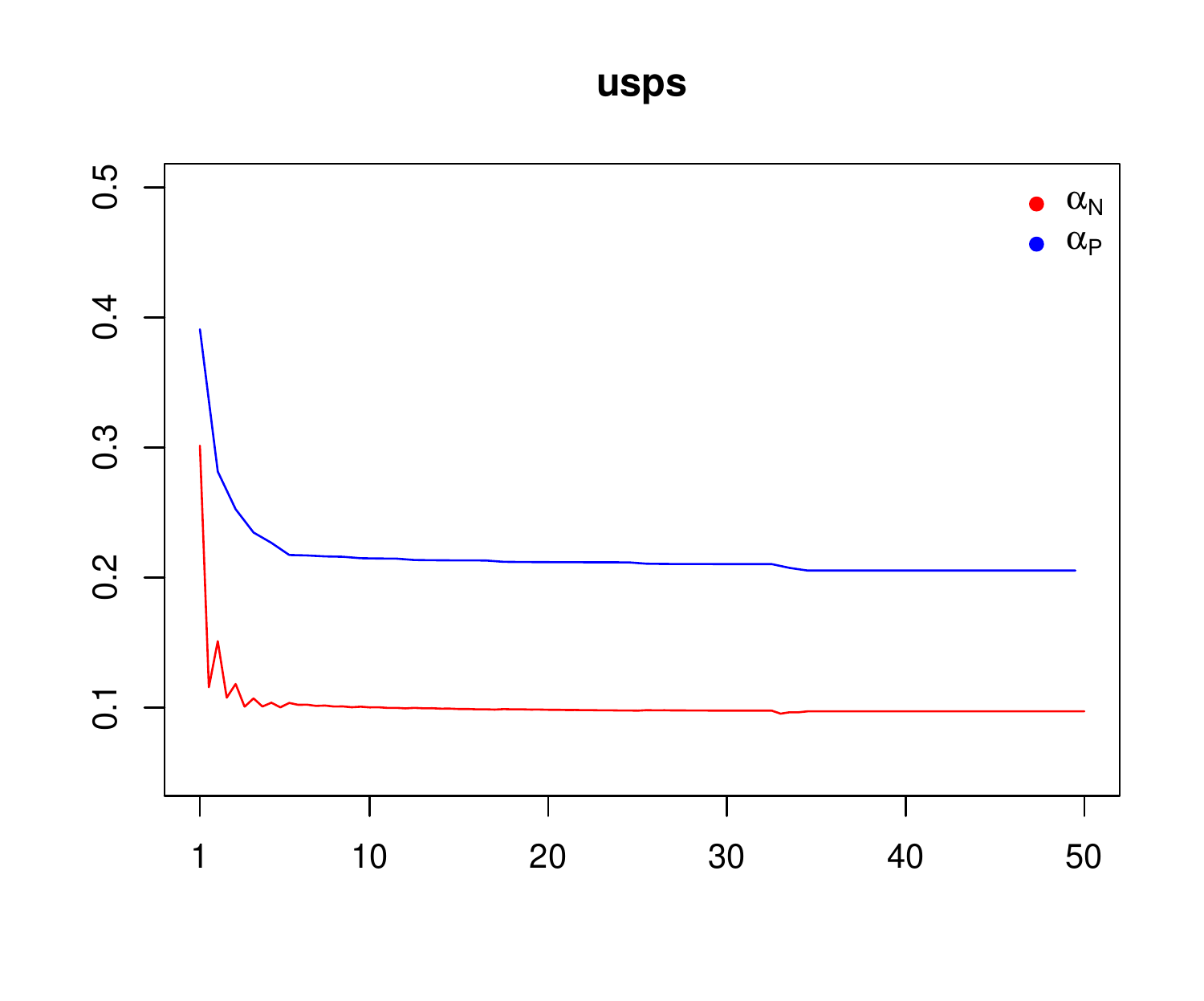}
}
\quad
\subfigure[covtype]{\includegraphics[width=4cm]{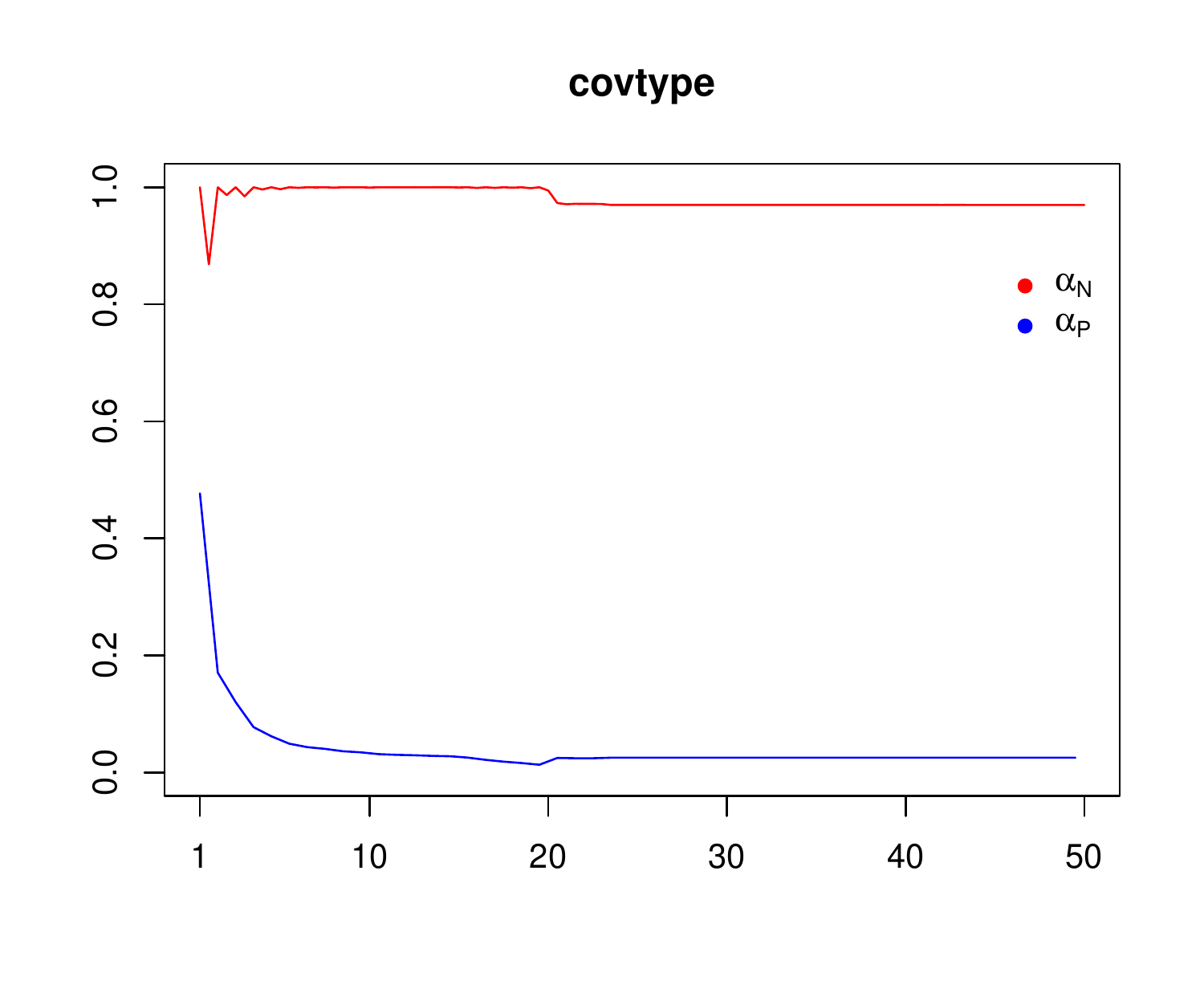}
}
\quad
\subfigure[ijcnn]{\includegraphics[width=4cm]{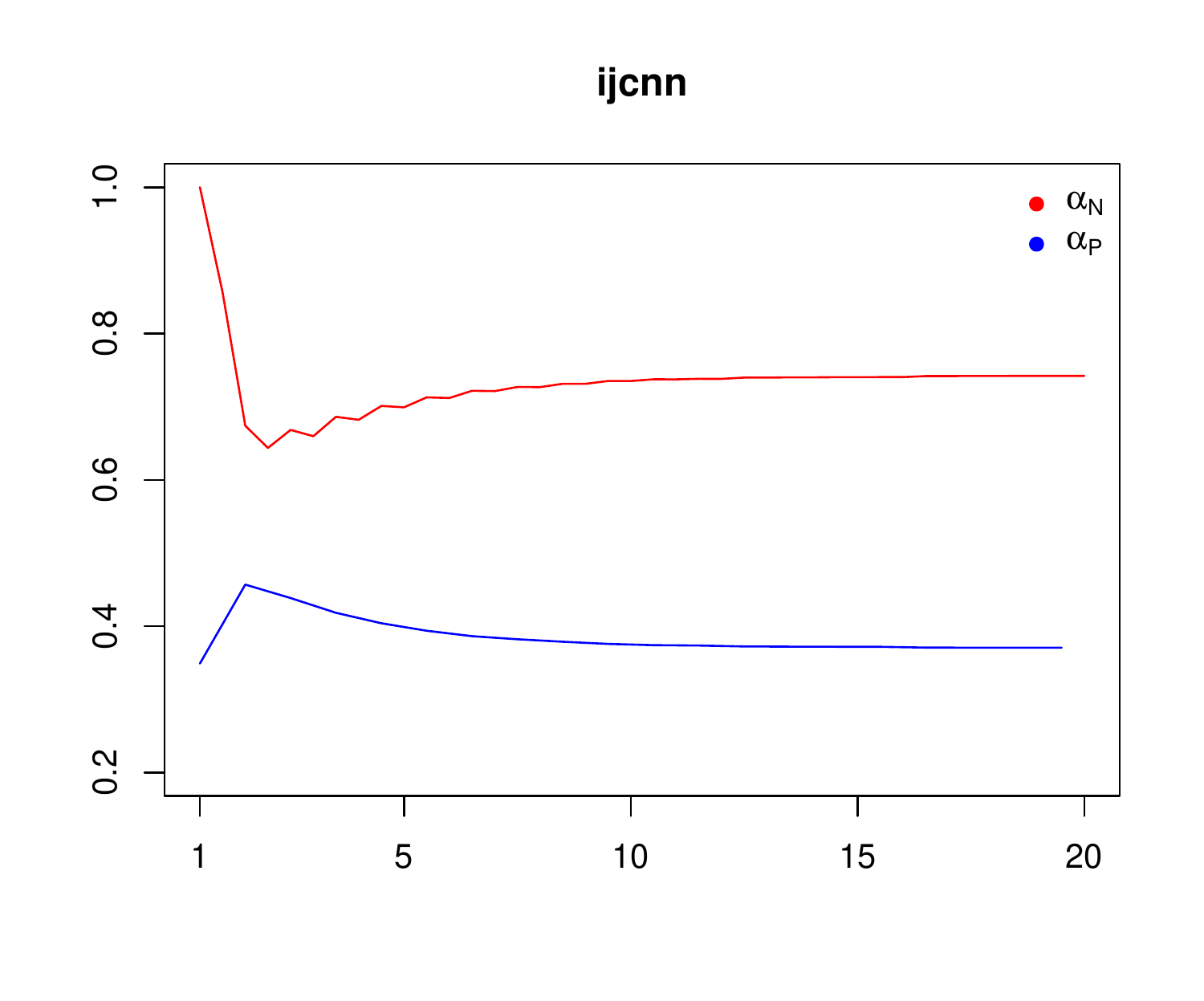}
}
\quad
\subfigure[skin]{\includegraphics[width=4cm]{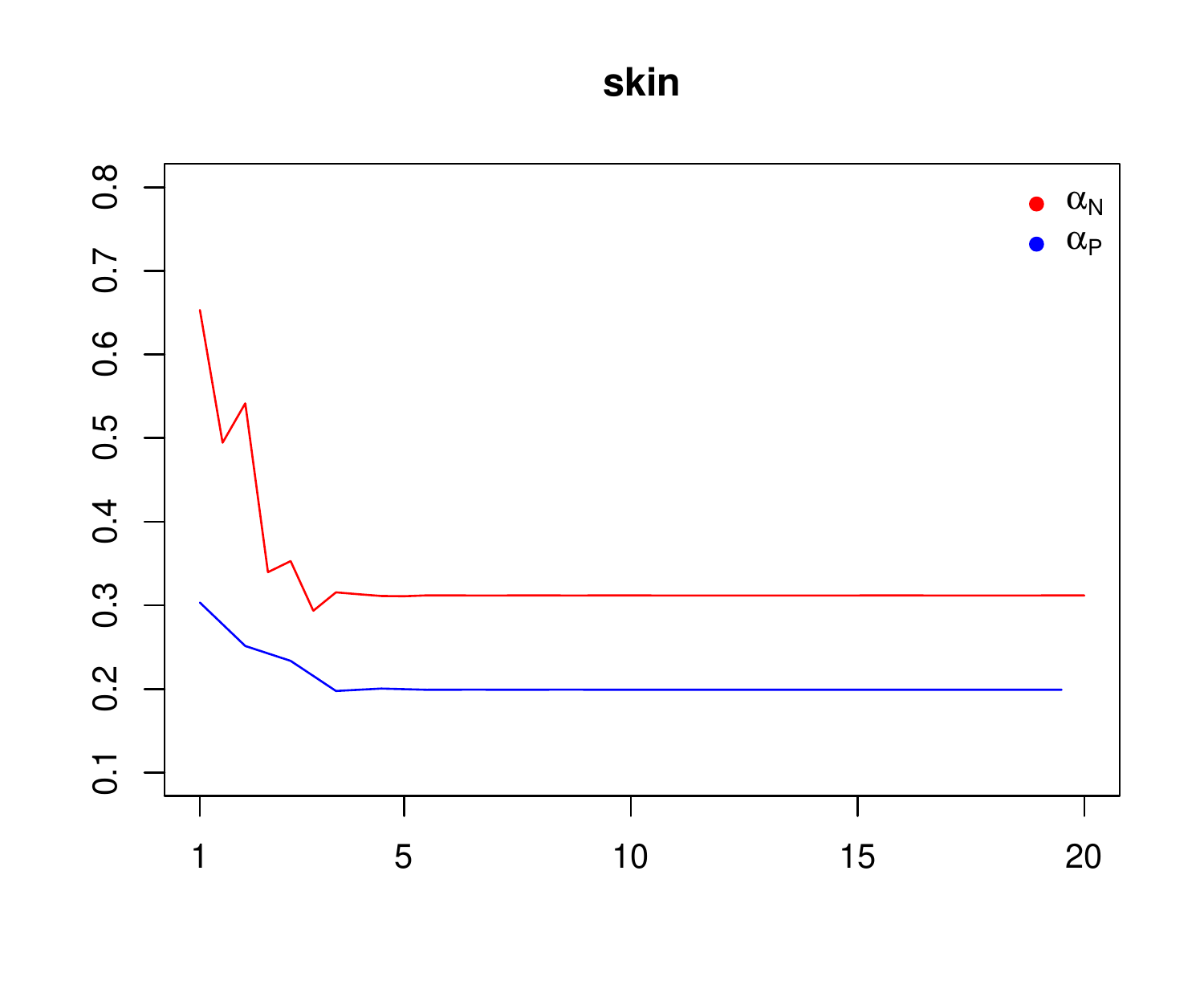}
}
\quad
\subfigure[sensorless]{\includegraphics[width=4cm]{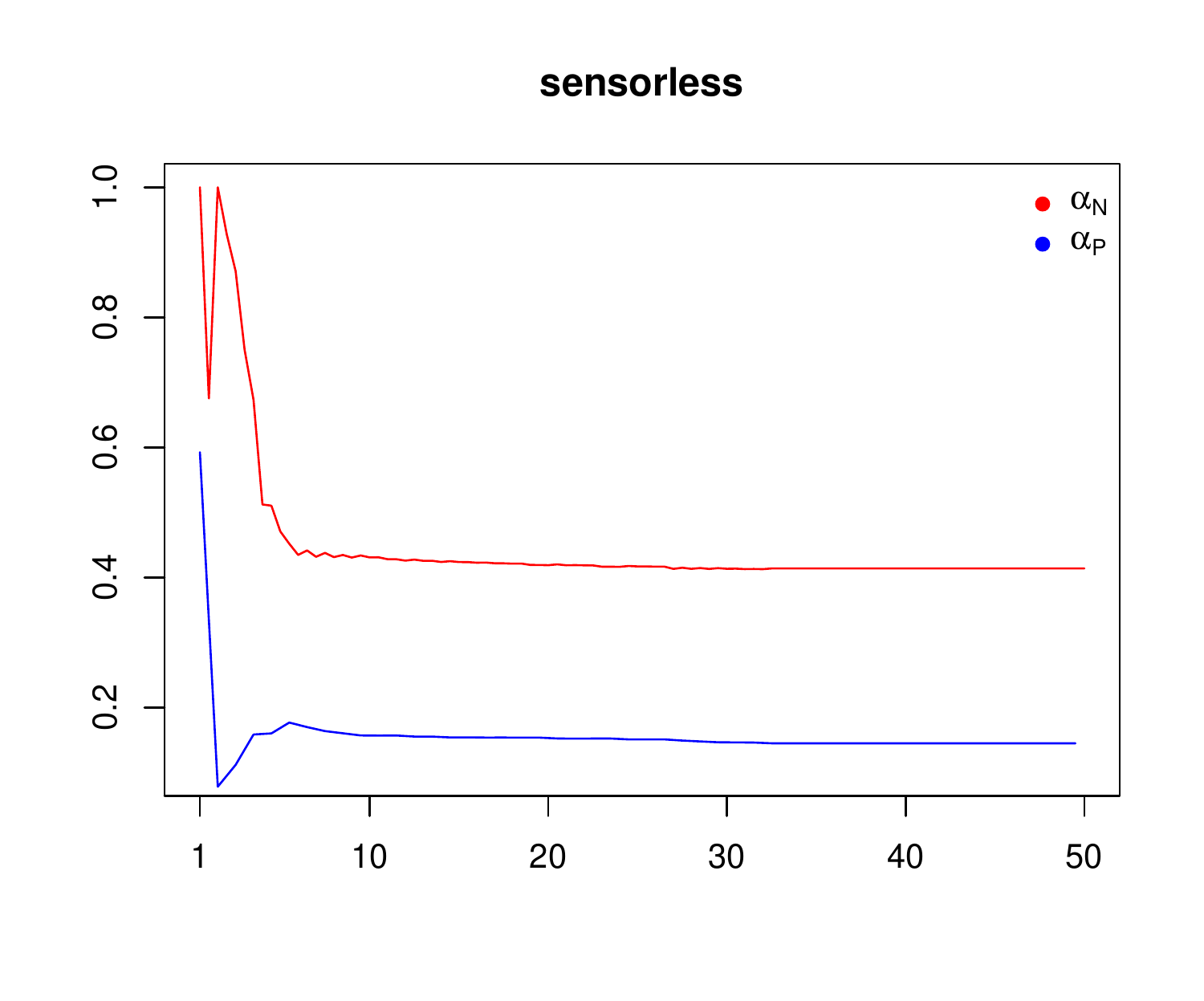}
}
\caption{The value of $\alpha_P$ (red line) and $\alpha_N$ (blue line) obtained as a function of iterations.
Both of them are convergent for each dataset.
}\label{fig:alpha}
\end{figure*}

\begin{table*}
\caption{The running time (in seconds) with different algorithms for imbalanced classification problems.}
\label{tab:running time-2}
\begin{center}
\begin{small}
\begin{sc}
\resizebox{\textwidth}{!}{
\begin{tabular}{llllllllllll}
\toprule
Dataset    & MPMF   & Plug-in & SVM   & ROS    & SMOTE  & ADASYN & RUS    & CC      & BRF    & RUSBoost & KMPMF \\
\midrule
letter     & 0.1878 & 0.1419 & 0.1346 & 0.9458 & 0.9115 & 1.1413 & 0.0151 & 58.796  & 1.0127 & 2.0387 & 6.0161 \\
breast     & 0.0550 & 0.0021 & 0.0018 & 0.0039 & 0.0047 & 0.0079 & 0.0024 & 2.9232  & 0.3917 & 0.7886 & 1.2615 \\
segment    & 0.2984 & 0.0280 & 0.0152 & 0.0739 & 0.0710 & 0.0985 & 0.0074 & 6.0516  & 0.6517 & 1.0120 & 5.6141 \\
usps       & 1.4405 & 24.088 & 0.7531 & 2.5059 & 2.3554 & 5.9611 & 0.1479 & 124.67  & 1.9727 & 10.271 & 18.138 \\
covtype    & 0.3504 & 9.4688 & 33.317 & 41.966 & 1201.6 & 2820.2 & 21.513 & -       & 163.56 & 290.15 & 457.27 \\
ijcnn      & 0.1082 & 0.3404 & 0.9117 & 2.5844 & 2.9956 & 8.4936 & 0.1421 & 1923.4  & 5.1220 & 10.891 & 28.282 \\
skin       & 0.2455 & 0.6936 & 3.8266 & 7.8850 & 8.1759 & 8.9347 & 1.7048 & -       & 17.873 & 35.793 & 22.468 \\
sensorless & 0.5073 & 7.9259 & 4.6929 & 10.856 & 12.336 & 14.781 & 0.6054 & 1021.8  & 4.0356 & 12.928 & 36.127 \\
\bottomrule
\end{tabular}
}
\end{sc}
\end{small}
\end{center}
\end{table*}

\begin{table*}
\caption{The value of the F1-measure obtained with different algorithms for imbalanced classification problems.
}\label{tab:measure-2}
\begin{center}
\begin{small}
\begin{sc}
\resizebox{\textwidth}{!}{
\begin{tabular}{llllllllllll}
\toprule
Dataset    & MPMF   & Plug-in & SVM   & ROS    & SMOTE  & ADASYN & RUS    & CC     & BRF & RUSBoost & KMPMF \\
\midrule
letter     & 0.5361 & 0.6122 & 0.4539 & 0.4188 & 0.4252 & 0.3338 & 0.4083 & 0.4035 &  0.7548 & 0.6021 & 0.8934 \\
breast     & 0.9676 & 0.9467 & 0.9477 & 0.9514 & 0.9480 & 0.9524 & 0.9540 & 0.9477 &  0.9620 & 0.9319 & 0.9772 \\
segment    & 0.8516 & 0.8899 & 0.7995 & 0.7870 & 0.7958 & 0.7640 & 0.7579 & 0.7566 &  0.9176 & 0.8982 & 0.9105 \\
usps       & 0.8533 & 0.8396 & 0.8745 & 0.8442 & 0.8547 & 0.8305 & 0.7887 & 0.7663 &  0.8753 & 0.7813 & 0.9011 \\
covtype    & 0.7105 & 0.7091 & 0.6783 & 0.7107 & 0.7120 & 0.7016 & 0.7106 & -      &  0.9506 & 0.7181 & 0.6718 \\
ijcnn      & 0.5524 & 0.5967 & 0.4178 & 0.5136 & 0.5228 & 0.5048 & 0.5212 & 0.2913 &  0.7693 & 0.6411 & 0.6269 \\
skin       & 0.8839 & 0.8442 & 0.8213 & 0.8634 & 0.8642 & 0.8344 & 0.8629 & -      &  0.9979 & 0.9167 & 0.9938 \\
sensorless & 0.6009 & 0.5860 & 0.2150 & 0.5351 & 0.5351 & 0.4641 & 0.5154 & 0.4857 &  0.9842 & 0.9681 & 0.9198 \\
\bottomrule
\end{tabular}
}
\end{sc}
\end{small}
\end{center}
\end{table*}

\section{Conclusion}\label{sec7}

In many real-world problems, only the mean and covariance matrix but not the true distribution of data are known 
in advance. To address this issue, \cite{Lanckriet2002} proposed the minimax probability machine (MPM) based on 
the mean and covariance matrix of data and the accuracy rate. However, the accuracy rate is not an appropriate 
metric for imbalanced classification problems. In this paper, we extended MPM to deal with imbalance classification 
problems based on some non-decomposable performance measures including the $F_\beta$ measure 
for the first time. To solve the new model effectively, we derived its equivalent minimization formulation in terms 
of a polynomial combination of FPR and FNR, which is then solved by using the alternating descent method.
The kernel trick is also used to obtain a nonlinear $F_\beta$ maximization.
The advantage of our method is that it makes use of only the mean and covariance matrix of data and thus 
is independent of the training data size, so our method is also appropriate for large-scale problems.
In addition, our method has no hyper-parameters to choose.
Experiments on both the synthetic dataset and real-world benchmark datasets have been presented to
illustrate the effectiveness of our method.

It is noted that certain online learning algorithms have also been studied recently for
the $F_\beta$-measure \cite{Narasimhan2015,Busa-Fekete2015,Liu2018} to handle large-scale problems.
As an ongoing project, we are currently developing certain online learning algorithms based on non-decomposable
performance measures including the $F_\beta$-measure which can deal with imbalanced large-scale problems.
In addition, deep networks have recently been applied to non-decomposable measure maximization problems \cite{Sanyal2018}.
Note that our method is not very efficient in dealing with very high-dimensional imbalanced classification problems 
due to the large storage requirement to store the estimator of the true covariance matrix of data. 
One way to overcome this difficulty is to use a dimension reduction method to reduce the feature dimensions of data.
We hope to study these topics in the near future.

\bibliography{mpmf}
\end{document}